\documentclass[11pt]{article}

\usepackage[preprint]{acl}

\usepackage{times}
\usepackage{latexsym}
\usepackage{amsmath}

\usepackage{amsmath, amsthm, amssymb}
\newtheorem{proposition}{Proposition}
\newtheorem{remark}{Remark}

\usepackage{hyperref}
\usepackage{url}
\usepackage{booktabs}
\usepackage{amsfonts}
\usepackage{nicefrac}
\usepackage{microtype}
\usepackage{xcolor}
\usepackage{marvosym}
\usepackage{enumitem}
\usepackage{bm}
\usepackage{multirow}
\usepackage{pifont}
\usepackage{algorithm}
\usepackage{algorithmic}
\usepackage{wrapfig}
\usepackage{colortbl}
\usepackage{amsmath}
\usepackage{amssymb}
\usepackage{mathtools}
\usepackage{amsthm}
\usepackage[most]{tcolorbox}
\usepackage{graphicx}
\usepackage{hyperref,xcolor}
\definecolor{midnightgreen}{rgb}{0.0, 0.29, 0.33}
\definecolor{deepgreen}{HTML}{055c29}
\definecolor{deeppurple}{HTML}{7030a0}
\definecolor{deepblue}{HTML}{171d91}
\definecolor{brown}{HTML}{843c0c}
\definecolor{shadered}{HTML}{ffe5e5}
\definecolor{shadegreen}{HTML}{e5f7ed}
\definecolor{msftBlack}{RGB}{0,0,0}
\definecolor{lightred}{RGB}{255,163,163}
\definecolor{deepred}{RGB}{153,0,0}

\definecolor{barblue}{RGB}{90,120,180}
\definecolor{barorange}{RGB}{225,124,5}

\newcommand{\methodname}{GiGPO}

\definecolor{softblue}{RGB}{30, 90, 160}
\definecolor{softbluebg}{HTML}{e8f0fb}

\usepackage{float}
\usepackage{tikz}
\usetikzlibrary{shapes.geometric, arrows.meta, positioning, calc}

\usepackage[T1]{fontenc}

\usepackage[utf8]{inputenc}

\usepackage{microtype}

\usepackage{inconsolata}

\usepackage{graphicx}

    \title{Productive State-Space Exploration for Multi-turn RL in LLM Agents}

\author{Mohammed Haroon Dupty$^{\dagger}$ \\
  Institute of Advanced Computing \\ and Intelligence, A*STAR\\
  \texttt{haroon@u.nus.edu} \\\And
  Yong Liang Goh \\
  {Grabtaxi Holdings Pte Ltd} \\
  \texttt{gyl@u.nus.edu} \\\AND
  Song Guo \\
  The Hong Kong University of \\ Science and Technology \\
  \texttt{songguo@cse.ust.hk} \\\And
  Wee Sun Lee \\
  School of Computing, \\ National University of Singapore \\
  \texttt{leews@comp.nus.edu.sg} \\}

\begin{document}
\maketitle
{\let\thefootnote\relax\footnotetext{$^{\dagger}$Major part of this work was done while at the National University of Singapore.}}

\begin{abstract}
Reinforcement Learning (RL) for Large Language Model (LLM) agents in multi-turn, long-horizon environments remains challenging due to inefficient action selection and repetitive behavioral patterns. In particular, agents often become trapped in high-confidence yet unproductive behaviors, repeatedly executing invalid actions or entering action loops that hinder task completion. We observe that these frequently violate environment preconditions, revealing a mismatch between the agent's representation of the internal state and the dynamics of the underlying environment. In addition, repetitive high-confidence actions lead to cyclical behaviors that provide little to no progress towards the goal.
We argue that addressing these challenges requires more \emph{productive task-aligned exploration} to improve the agent's understanding of environment dynamics. Motivated by intrinsic curiosity-driven exploration, we propose a \emph{conditional state-space exploration} framework that balances task completion with intrinsic information-seeking behavior. Specifically, we introduce an intrinsic reward that penalizes actions whose observation duplicates prior state memory. This provides the benefit of reducing redundant behaviors and encouraging state-space coverage without relying on a general information-theoretic measure. Strong experimental results demonstrate that our approach substantially improves policy learning in multi-turn environments, mitigating invalid actions and repetitive loops.
\end{abstract}

\section{Introduction}

Reinforcement learning (RL) has become an important training paradigm for large language model (LLM) agents that must operate over multiple turns in partially observable, long-horizon environments~\citep{carta2023grounding,shani2024multi,guan2026evaluating}. In contrast to single-step prediction tasks, these agents must select actions sequentially, update internal memory, and adapt to changing states while receiving sparse and delayed feedback. Recent work has shown that RL is increasingly relevant for improving reasoning, agentic control, and long-horizon decision making in LLMs, but it has also revealed a persistent difficulty: agents often become stuck in repetitive, high-confidence behaviors that are locally plausible yet unproductive for the task \citep{liu2025reinforcement,zhang2025survey,xi2025agentgym}.

In such settings, effective learning depends not only on exploiting promising actions but also on exploring the environment in a way that improves the agent's understanding of its dynamics. In practice, agents often overcommit to high-confidence actions that are locally plausible but globally unhelpful, leading to invalid interactions or repetitive loops that consume budget without advancing the task. Such failures, as shown in Figure \ref{fig:intro}, typically reflect a mismatch between the agent's internal state representation and the true environment dynamics, especially when the agent lacks sufficient feedback about which actions are feasible or informative. This motivates the need for a form of \emph{productive state-space exploration} that helps the agent learn environment dynamics while remaining aligned with task completion.

\begin{figure*}[t]
\centering
  \includegraphics[width=0.49\linewidth]{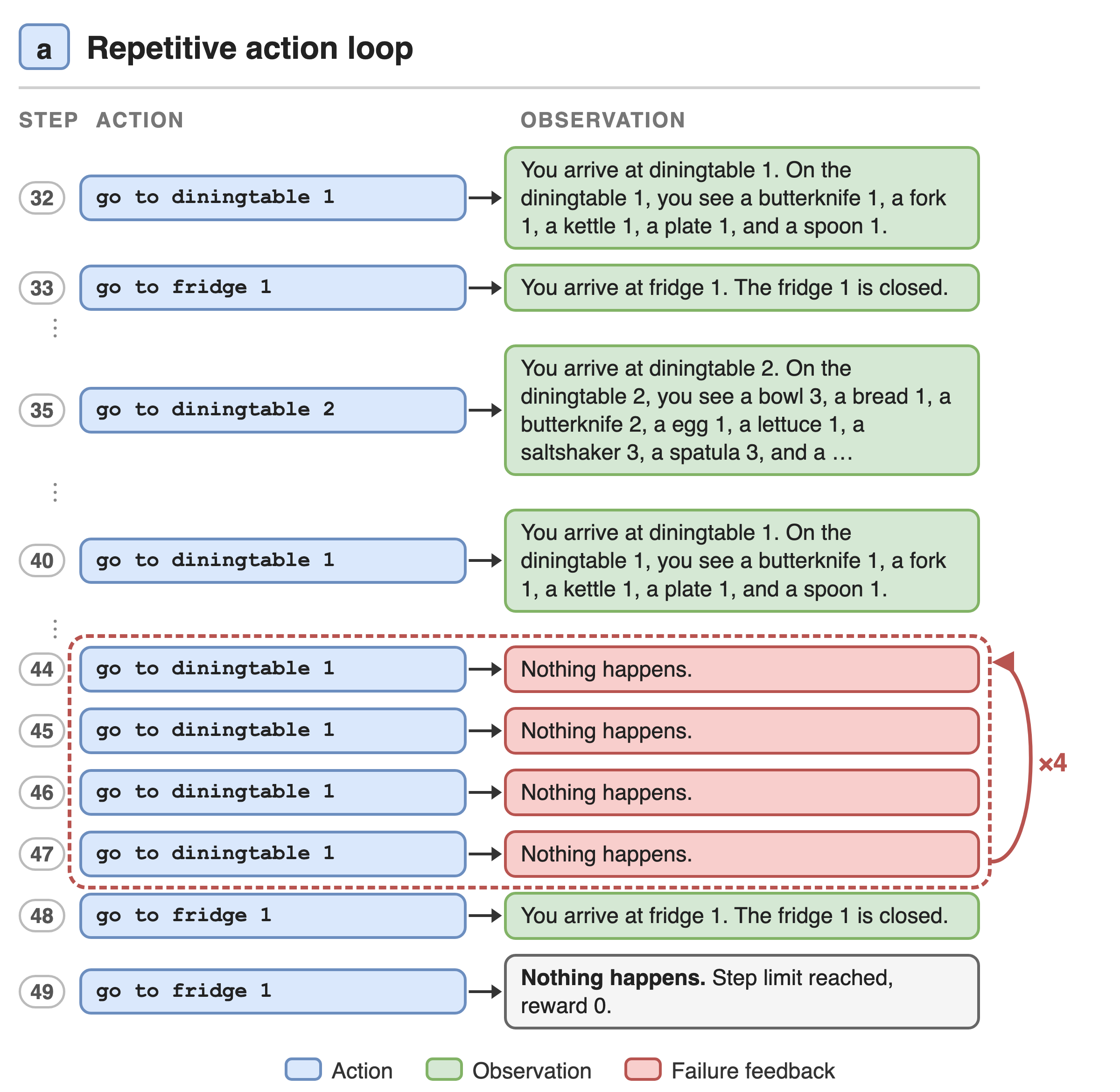} \hfill
  \includegraphics[width=0.49\linewidth]{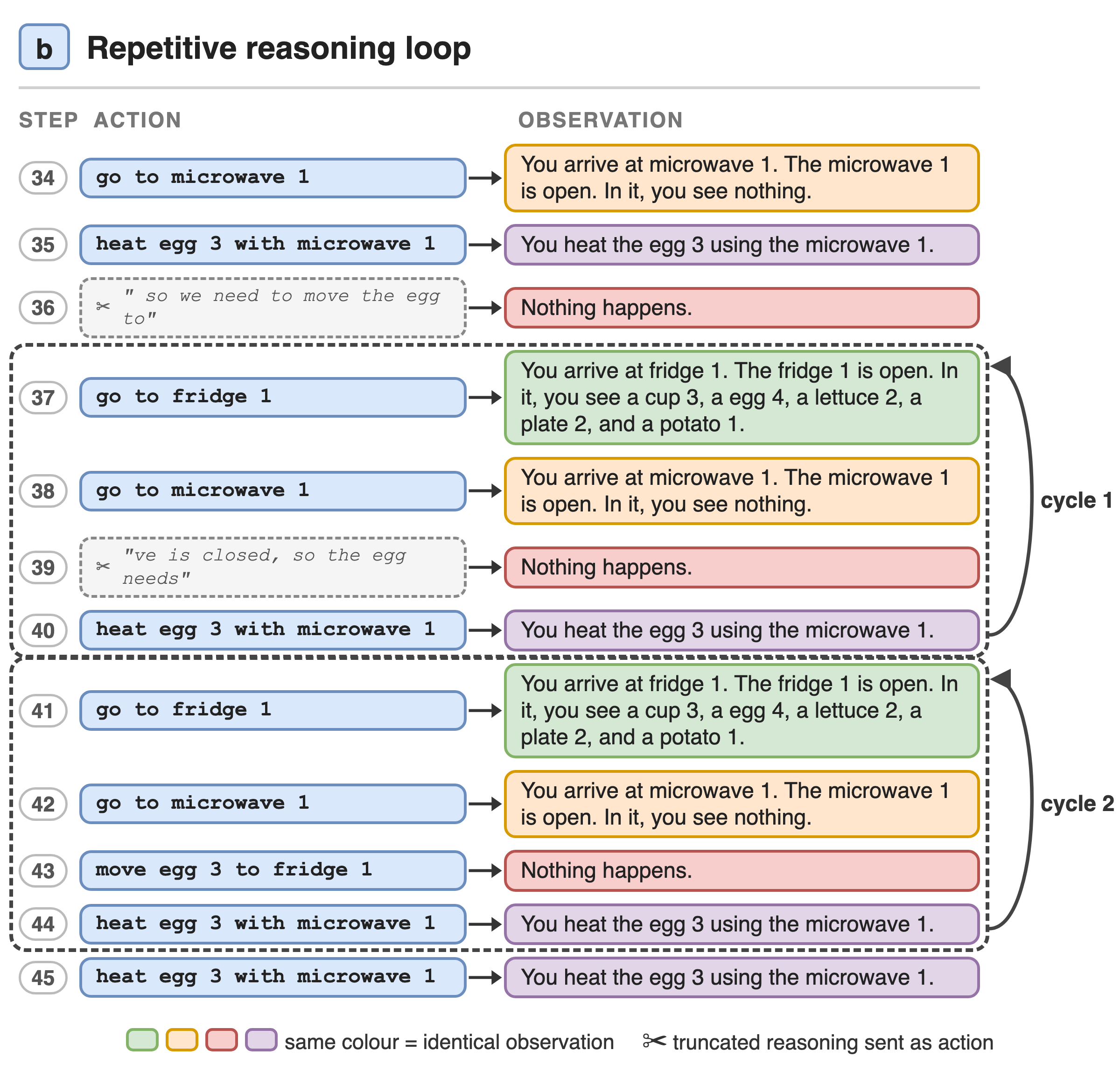}
  \caption{Failure modes of an RL-trained LLM agent in \texttt{ALFWorld}. (a)~\emph{Repetitive action loop.} The agent repeats an invalid action and receives the same failure feedback until the step budget runs out. (b)~\emph{Repetitive reasoning loop.} The agent's reasoning ignores the current observation, and the agent cycles between the fridge and the microwave without completing the task. In both cases the agent's internal state diverges from the environment dynamics, and the looping steps consume interaction budget while revealing no new observation. These are the steps our intrinsic reward penalizes.}
  \label{fig:intro}
\end{figure*}

A natural way to improve exploration is to encourage the agent to visit a broader range of states. In principle, increasing the entropy of the explored state space can improve coverage, reduce mode collapse, and expose the agent to more informative trajectories. However, directly maximizing exploration entropy can also be problematic. If encouraged too strongly, the agent may drift toward diverse but task-irrelevant behavior, thereby sacrificing progress on the underlying objective. In long-horizon environments, this tension between exploration and exploitation is especially important: \emph{exploration should broaden the agent's experience, but it should not pull the policy away from useful action sequences that lead to task completion.}

This paper pursues productive exploration through \emph{intrinsic curiosity-driven rewards}, with a conservative formulation that preserves task alignment. Rather than positively rewarding state-space entropy increase, we penalize any step whose observation has already appeared in the episode. The intuition is that unhelpful actions, such as invalid steps or redundant repetitions, consume interaction budget and policy probability without revealing anything new about the environment. They can be detected simply by checking whether each new observation adds any new information to the agent's memory.

Our formulation reframes entropy increase as a consequence of avoiding unproductive actions rather than as an objective to be optimized directly. The penalty improves within-episode state-space coverage, while the task reward keeps exploration tied to task completion. We further formalize this argument by showing that the cumulative penalty measures how many distinct observations an episode visits, and that it bounds the episode's visitation entropy (Section~\ref{sec:method}). A single-step analysis further shows that novel observations always increase entropy, while revisits of frequent observations decrease it.

Our approach is consistent with the broader literature on intrinsic motivation in reinforcement learning, where curiosity, novelty seeking, and uncertainty reduction have been used to improve learning under sparse rewards \cite{chentanez2004intrinsically,aubret2019survey, gao2025navigate}. It also aligns with recent efforts to train long-horizon LLM agents, where robustness depends not only on reasoning ability but also on the capacity to avoid repetitive or uninformative interactions with the environment \cite{kalyan2025reinforcement}. Within this setting, penalizing unhelpful actions provides a simple compromise between unrestricted entropy maximization and overly conservative exploitation. The resulting behavior promotes broader state coverage without losing contact with the task. 

In summary, we propose a \emph{conditional intrinsic reward} framework that supports \emph{productive exploration} by penalizing actions that yield no new informative observation. Our key argument is that entropy expansion is useful, but only when it is induced indirectly through the suppression of redundant or uninformative behavior.
This allows the agent to explore more broadly while maintaining a strong coupling to task progress.
We evaluate our proposed exploration strategy on two prominent multi-turn benchmark environments: \texttt{ALFWorld} and \texttt{WebShop}. Empirical results show that our approach surpasses recent baselines in performance metrics of success rates and significantly mitigates both invalid and repetitive action rates.

\section{Related Work}

\subsection{Prompting-Based LLM Agents}

Early LLM agent frameworks rely on prompting for multi-step reasoning and action selection without updating model weights. ReAct \cite{yao2022react} interleaves reasoning (Thought) and actions (Act), outperforming pure reasoning or action prompting on tasks such as ALFWorld and WebShop. Reflexion \cite{shinn2023reflexion} extends this with verbal reinforcement, storing self-reflections from failed episodes in episodic memory to guide future decisions. Although ReAct and Reflexion improve performance on multi-turn tasks, both methods depend heavily on self-evaluation quality, which degrades for smaller LLMs, and cannot structurally suppress invalid actions since prompting alone cannot alter action priors.

\subsection{Reinforcement Learning for LLM Agents}
Recently, reinforcement learning (RL) has emerged as a principled approach for post-training LLMs, with growing focus on agent-based LLMs in multi-turn environments. PPO \cite{schulman2017proximal} and its critic-free variant GRPO \cite{shao2024deepseekmath}, which estimates baselines from group rollout scores, are widely adopted for RL-based LLM training. RLOO \cite{kool2019buy, ahmadian2024back} further reduces variance without a learned critic, and AGILE \cite{feng2024agile} applies PPO to LLM agents using a four-module architecture with proactive expert consultation, significantly outperforming prompting on long-horizon QA tasks. Retrospex \cite{xiang2024retrospex} combines an offline RL critic with LLM action likelihood through dynamic rescoring, improving performance on ALFWorld and WebShop over prompting baselines.

GiGPO~\cite{feng2026group} captures global trajectory quality and group-level effectiveness without additional rollouts, providing rich signals for LLM agent learning. Recently, ProxMO \cite{fang2026proximity} presents a framework to address GRPO's context-agnostic normalization and GiGPO's discrete state-grouping limitation. They explore the provision of gradient intensities based on task difficulty, and replace hard boundary partitioning with continuous semantic similarity weighting, so as to provide richer reward signals that are less task-agnostic. In totality, these works establish multiple avenues to improve vanilla RL for multi-turn LLM agents, such as stable baselines and supporting tools to provide more information about the world. However, they do not directly address the exploration failure modes we identify in this work. Specifically, we explore the notion of invalid actions that arise from environment precondition mismatches, resulting in action loops driven by high-confidence but unproductive behaviors. Our approach adds a conditional intrinsic reward that explicitly penalizes such behaviors during RL training.

Recent RL methods for multi-turn LLM agents focus on improving credit assignment, exploration, and training stability. ArCHer \cite{zhou2024archer} introduces a hierarchical RL framework combining off-policy temporal difference learning at the utterance level with on-policy policy gradients at the token level, substantially improving sample efficiency over PPO while emphasizing information-seeking behavior in multi-turn settings. Tree-GRPO \cite{ji2025treesearch} improves training efficiency by replacing chain-based rollouts with tree-structured sampling, enabling step-level supervision from outcome rewards alone. Addressing exploration instability, EPO \cite{xu2025epo} identifies \emph{exploration-exploitation cascade failure} in shared multi-turn policies and mitigates it via trajectory-level entropy regularization and entropy smoothing, yielding gains on ScienceWorld and ALFWorld. Similarly, RAGEN \cite{wang2025ragen} studies multi-turn RL failure modes, identifying an ``Echo Trap'' caused by reward variance cliffs and gradient spikes, and proposes trajectory filtering with a critic for stabilization. While these works improve \emph{how} agents learn---through efficient credit propagation, entropy stabilization, or robust optimization---our approach is complementary, focusing on \emph{what} to explore by penalizing invalid or repetitive actions that fail to expand state-space knowledge, without itself performing credit assignment across turns.

\subsection{Intrinsic Motivation and Exploration in RL}

Intrinsic motivation has long been used to improve exploration under sparse rewards by encouraging state diversity and novelty. Curiosity-based methods like~\citet{pathak2017curiosity,burda2018large} reward novel states or high prediction errors, while count-based exploration formalizes revisitation penalties through information gain via KL-divergence between prior and posterior state beliefs~\citep{bellemare2016unifying}. Extending state-entropy exploration to task-reward settings,~\citet{kim2023accelerating} propose value-conditional state entropy (VCSE), addressing the tendency of naive entropy maximization to overexplore low-value regions by maximizing entropy within value partitions.~\citet{mutti2021task,ashlag2026state} further provide a theoretical foundation for state entropy regularization, showing that it improves robustness and promotes broader coverage across multiple viable trajectories, unlike policy entropy which mainly injects noise along a dominant path. These works motivate our approach -- by targeting state-space entropy, our conditional intrinsic penalty discourages revisiting already-observed states.
Additionally, by tying the penalty to information gain rather than raw visitation counts, we ensure that only transitions that fail to expand the agent's knowledge are penalized, preserving the pull of the task reward toward goal-relevant states.

In LLM-based RL, intrinsic signals have recently been explored to complement sparse task rewards. MIRA \cite{nourzad2026mira} improves exploration by using memory graphs from agent trajectories and semantic reasoning to shape PPO advantages, mitigating poor early critic calibration. Similarly, our approach leverages accumulated experience for exploration, but instead defines an information-gain-based intrinsic signal that penalizes observations failing to expand state knowledge. This is simpler to compute, and is motivated by Shannon entropy (Proposition~\ref{proposition_coverage}), while being aligned with repetitive and invalid behaviors in text-based environments.
Further recent work on moral intrinsic rewards \cite{tennant2025moral} has shown that intrinsic signals can shape agent behavior toward desired norms without task-specific reward engineering, reinforcing the broader applicability of intrinsic motivation for LLM alignment. Our contribution fits within this emerging direction: intrinsic rewards as a practical mechanism for shaping LLM agent behavior in ways that extrinsic rewards alone cannot achieve.

\section{Methodology}
\label{sec:method}
\subsection{Problem Setup}
We consider reinforcement learning for LLM agents operating in partially observable, long-horizon interactive environments. At each timestep \(t\), the agent is in state \(s_t\), takes an action \(a_t \in \mathcal{A}\), and receives an observation \(o_t\) from the environment. We define the agent's state \(s_t\) as the set of all observations received so far along the trajectory, i.e., the information the agent has accumulated about the environment up to time \(t\). The objective is to learn a policy \(\pi(a_t \mid s_t)\) that maximizes expected cumulative reward while completing the task within a finite horizon \(T\).

Unlike standard short-horizon decision problems, this setting introduces two recurring failure modes. First, the agent may produce \emph{invalid actions} whose preconditions are not satisfied, causing the transition to fail or yielding no meaningful progress. Second, the agent may enter \emph{action loops}, repeatedly selecting a small set of actions that do not improve task state or move the agent closer to completion. Both issues reduce learning efficiency and often reflect an imperfect internal model of the environment dynamics. In practice, these failures are especially harmful in LLM agents because they waste interaction budget and reinforce unproductive behavior patterns.

\subsection{State Representation}
To support long-horizon decision making, we define the agent's state as the accumulated observation history:
\[
s_t = \{o_1, o_2, \dots, o_t\}.
\]
This representation captures the information available to the agent at time \(t\). Instead of relying only on the most recent observation, the agent maintains a memory of the full trajectory of \emph{unique observations}, enabling it to reason about what has already been explored and what remains uncertain. This accumulated state serves as the basis for evaluating whether a new observation provides additional information about the environment.

At each step, the agent observes the current state \(s_t\), selects an action \(a_t\), and receives the next observation \(o_{t+1}\). The state is then updated by appending the new observation:
\[
s_{t+1} = s_t \cup \{o_{t+1}\}.
\]
The intrinsic reward is computed from the change in information content between \(s_t\) and \(o_{t+1}\). If the new observation does not improve the agent's knowledge of the environment, a negative intrinsic reward is assigned. This signal is combined with the task reward and used to update the policy.

\subsection{Conditional State-Space Exploration}
We propose a conditional state-space exploration mechanism that penalizes unhelpful transitions. After taking action \(a_t\), the agent receives a new observation \(o_{t+1}\), and the information gain is computed by comparing this new observation with the previous state \(s_t\). Let \(I(\cdot)\) denote a measure of information content or entropy associated with the agent's accumulated state representation. We define the change in information as
\[
\Delta I(s_t, o_{t+1}) = I(s_t \cup \{o_{t+1}\}) - I(s_t).
\]
This quantity measures whether the new observation expands the agent's knowledge of the environment.

The intrinsic reward is then defined as:
\begin{equation}
r_t^{\text{intrinsic}} =
\begin{cases}
-\gamma & \text{if } \Delta I(s_t, o_{t+1}) \leq 0, \\
0 & \text{otherwise},
\end{cases}
\label{eq:main}
\end{equation}

where \(\gamma > 0\) controls the penalty magnitude. This formulation penalizes actions that do not lead to a meaningful increase in the agent's information about the environment, thereby discouraging redundant or cyclic behavior. 

The formulation is general, and $I(\cdot)$ can be any measure of the information in the agent's memory. In practice, we use a simple instantiation, $I(s)=\log|s|$, with new observations compared by string matching. Then $\Delta I(s_t,o_{t+1})\le 0$ iff $o_{t+1}\in s_t$, and Equation~\ref{eq:main} reduces to a duplicate-observation penalty (Algorithm~\ref{alg:reward}). Exact matching suffices in many settings, including \texttt{ALFWorld}, where environment feedback is drawn from a fixed set of templated observations, so identical states produce identical text. It errs only toward under-penalization, since it may miss near-duplicates but never penalizes a genuinely new observation. Unlike similarity-based measures, it also requires no threshold and adds no hyperparameters beyond $\gamma$. Moreover, since the agent controls actions rather than observation text, it can avoid the penalty only by reaching a different state. More sophisticated measures, such as embedding similarity, are needed only in environments with free-form or noisy observations, where the same state can be described in different ways.

\begin{algorithm}[t]
\caption{Computation of the intrinsic reward}
\label{alg:reward}
\begin{algorithmic}[1]

\REQUIRE state memory $s_t=\{o_1,\dots,o_t\}$, new observation $o_{t+1}$, penalty weight $\gamma$
\IF{$o_{t+1} \in s_t$ }
    \STATE $r_t^{\text{intrinsic}} \gets -\gamma$
    \STATE $s_{t+1} \gets s_t$
\ELSE
    \STATE $r_t^{\text{intrinsic}} \gets 0$
    \STATE $s_{t+1} \gets s_t \cup \{o_{t+1}\}$
\ENDIF
\STATE \textbf{return} $r_t^{\text{intrinsic}}$, $s_{t+1}$
\end{algorithmic}
\end{algorithm}

\paragraph{How Does the Penalty Teach Preconditions?}
Let $\mathcal{F}(s)$ be the set of actions whose preconditions fail in state $s$. We assume that (A1) every $a \in \mathcal{F}(s)$ returns the same observation $o_\bot$, independent of $a$ and $s$, and (A2) a failed action leaves the state unchanged. Under A1, every precondition violation after the first in an episode yields $o_{t+1} = o_\bot \in s_t$ and is penalized, whichever action failed. The policy is thus pushed away from all precondition-violating actions, not only from repeating one. Under A2, avoiding them loses no information, so the penalty does not conflict with information gathering.

\subsection{Final Reward}
The final reward used for policy optimization combines the task reward and the intrinsic reward:
\[
r_t = r_t^{\text{task}} + \lambda r_t^{\text{intrinsic}},
\]
where \(\lambda \geq 0\) balances task performance and exploration pressure. The resulting objective is
\[
J(\pi) = \mathbb{E}_{\pi}\left[\sum_{t=0}^{T-1} \rho^t r_t\right],
\]
where \(\rho \in [0,1]\) is the discount factor. This objective encourages the agent to maximize task reward while avoiding actions that fail to improve state-space coverage.

\subsection{Entropy-Oriented Exploration Objective}
The motivation for this design is that directly maximizing state-space entropy can sometimes distract the agent from the task objective. An overly strong incentive for novelty may lead the policy to explore diverse but irrelevant trajectories. By contrast, penalizing non-informative actions creates a task-constrained pressure toward visiting more distinct observations within an episode. The policy is discouraged from remaining in low-diversity action cycles, but it is still constrained by the external task reward, which ensures that exploration remains useful for task completion.

This conditional structure is especially important in long-horizon environments, where the agent must balance exploration and exploitation over many sequential decisions. Actions that expand state coverage or contribute to task progress receive \emph{no} intermediate positive incentive. However, actions that actively fail to improve the agent's state knowledge are \emph{penalized}. As a result, the learned policy is pushed toward trajectories that are both diverse and productive.

This penalization mechanism encourages productive exploration by discouraging repeated observations and promoting broader coverage of the agent's experience. It is related to information-theoretic exploration, which uses visitation entropy to characterize exploration diversity. Below, we formalize this connection through single-step changes in visitation entropy and episode-level coverage.

Let $P_t$ denote the empirical observation-visitation distribution
over $o_1,\dots,o_t$, with $p_i^{(t)}=c_i^{(t)}/N$, where
$c_i^{(t)}$ counts occurrences of the $i$-th distinct observation,
$n=|s_t|$, and $N=\sum_{i=1}^{n}c_i^{(t)}=t\ge1$.
Its Shannon entropy is
$H(P_t)=-\sum_{i=1}^{n}p_i^{(t)}\log p_i^{(t)}$.
Unlike the memory-based measure $I(s_t)=\log|s_t|$,
$H(P_t)$ accounts for observation frequencies.

\begin{proposition}
\label{proposition_1}
For the distribution $P_t$ defined above, let $n=n_t$ be the
number of distinct states observed by time $t$. States are identified
with observations in the memory $s_t$. Then,
\begin{enumerate}
\item If the new observation introduces a previously unseen
state, then $H(P_{t+1})>H(P_t)$.

\item If the new observation repeats an already-visited state
$k$ satisfying $p_k^{(t)}\ge e^{-H(P_t)}$, then
$H(P_{t+1})\le H(P_t)$.
\end{enumerate}
\end{proposition}

\begin{proof}
Let $\alpha=1/(N+1)$, where $N=t$ is the total count of observations.
The update induced by a new observation can be written as
\[
P_{t+1}
=
(1-\alpha)P_t+\alpha e,
\]
where $e$ is a point mass on the observed state.

\medskip
\noindent\textbf{Case 1: Introduction of a new state.}

Suppose the observation corresponds to a previously unseen state.
Then the support expands by one dimension and
\[
P_{t+1}
=
\bigl((1-\alpha)P_t,\alpha\bigr).
\]
The entropy of the updated distribution decomposes as
\[
H(P_{t+1})
=
(1-\alpha)H(P_t)+h(\alpha),
\]
where $h(\alpha) = -(1-\alpha)\log(1-\alpha)
-\alpha\log\alpha$ is the binary entropy term arising from introducing the new state. Therefore,
\[
H(P_{t+1})-H(P_t)
=
h(\alpha)-\alpha H(P_t).
\]
Since the support size before the update is at most $N$,
\[
H(P_t)\le \log N.
\]
Substituting $\alpha=1/(N+1)$ gives
\begin{align*}
h(\alpha)-\alpha H(P_t)
&\ge h(\alpha)-\alpha\log N\\
&=\alpha\log\frac{N+1}{N}
 -(1-\alpha)\log(1-\alpha)\\
&>0
\end{align*}
Hence $H(P_{t+1})>H(P_t)$.

\medskip
\noindent\textbf{Case 2: Repetition of an existing state.}

Suppose the observation revisits an already-seen state
$k\in\{1,\dots,n\}$. Then
\[
P_{t+1}
=
(1-\alpha)P_t+\alpha e_k,
\]
where $e_k$ is the point mass on state $k$. For $\lambda\in[0,1]$, define
\[
f(\lambda)
=
H\!\left((1-\lambda)P_t+\lambda e_k\right).
\]
If $P_t=e_k$, then $f(\lambda)=0$ for all
$\lambda\in[0,1]$. Otherwise, since Shannon entropy is strictly
concave on the simplex and $P_t\neq e_k$, $f$ is strictly
concave in $\lambda$.
Its right derivative at $\lambda=0$ is
\[
f'(0)
=
-H(P_t)-\log p_k^{(t)}.
\]
Thus,
\[
f'(0)\le0
\quad\Longleftrightarrow\quad
p_k^{(t)}
\ge
e^{-H(P_t)}.
\]
By concavity, a non-positive initial slope implies
\[
f(\lambda)\le f(0)
\]
for all $\lambda\in(0,1]$. If $P_t\neq e_k$, strict concavity gives
\[
f(\lambda)<f(0)
\qquad\text{for all }\lambda\in(0,1].
\]
Taking $\lambda=\alpha$ gives
\[
H(P_{t+1})
\le
H(P_t),
\]
with strict inequality unless $P_t=e_k$.
\end{proof}

\begin{remark}
Proposition~\ref{proposition_1} formalizes the intuition underlying our intrinsic
penalty at the level of a single visitation update. Introducing novel states strictly increases visitation
entropy, whereas repeatedly revisiting sufficiently frequent
states reduces entropy and concentrates probability mass.
Penalizing repeated observations therefore discourages
unproductive revisitation and is directionally aligned with
broader intrinsic exploration.
\end{remark}

\begin{proposition}[Whole-episode coverage and entropy]
\label{proposition_coverage}
Consider an episode with observations $o_1,\dots,o_L$ and
$D=\left|\{1\le t<L : o_{t+1}\in s_t\}\right|$ duplicate steps.
Then $|s_L|=L-D$, and the undiscounted cumulative intrinsic reward is
\[
R^{\text{int}}
=\sum_{t=1}^{L-1}r_t^{\text{intrinsic}}
=-\gamma D
=-\gamma(L-|s_L|).
\]
Thus, at fixed episode length $L$, maximizing $R^{\text{int}}$
maximizes within-episode coverage $|s_L|$. Moreover,
\[
H(P_L)\le \log(L-D),
\]
\[
H(P_L)\ge
\log L-\frac{D+1}{L}\log(D+1).
\]
The lower bound is attained, both bounds are decreasing in $D$,
and both equal $\log L$ if and only if $D=0$, equivalently
$R^{\text{int}}=0$.
\end{proposition}

{\begin{proof}
\textit{Coverage.}
$o_1$ is always novel, and each subsequent observation either
adds one element to the memory or is a duplicate. Therefore,
\[
|s_L|=1+(L-1-D)=L-D.
\]
By the definition of the intrinsic reward,
\[
R^{\text{int}}=-\gamma D=-\gamma(L-|s_L|).
\]

\textit{Upper bound.}
Entropy is at most the logarithm of the support size:
\[
H(P_L)\le\log|s_L|=\log(L-D).
\]

\textit{Lower bound.}
For positive integer counts $c_1,\dots,c_{L-D}$ summing to $L$,
\[
H(P_L)=\log L-\frac{1}{L}\sum_{i=1}^{L-D}c_i\log c_i.
\]
The sum $\sum_i c_i\log c_i$ is maximized when the counts are
$(D+1,1,\dots,1)$. Indeed, whenever two counts satisfy $a\ge b>1$,
moving one unit from $b$ to $a$ does not decrease the sum, since
$x\log x$ is convex. Repeating this operation leaves all but one
count equal to $1$. Therefore,
\begin{align*}
H(P_L)
&\ge
H\left(\frac{D+1}{L},\frac1L,\dots,\frac1L\right)\\
&=
\log L-\frac{D+1}{L}\log(D+1),
\end{align*}
with equality when the count vector is $(D+1,1,\dots,1)$,
up to permutation.

\textit{Monotonicity.}
The upper bound $\log(L-D)$ is strictly decreasing in $D$.
The lower bound is also strictly decreasing, since
$(D+1)\log(D+1)$ is strictly increasing for $D\ge0$.
Both bounds equal $\log L$ when $D=0$ and are strictly below
$\log L$ when $D\ge1$. Since $\gamma>0$,
$D=0$ if and only if $R^{\text{int}}=0$.
\end{proof}}

\begin{remark}
Proposition~\ref{proposition_coverage} establishes that, for a fixed episode length, maximizing the
cumulative intrinsic reward is equivalent to maximizing within-episode
coverage. It also bounds the visitation entropy in terms of the number
of duplicate steps, showing that fewer duplicates permit higher entropy.
However, we do note that maximizing coverage does not necessarily maximize entropy across trajectories, since entropy also depends on the frequency distribution of visited observations.
\end{remark}

\section{Experiments}
\label{sec:experiments}
\begin{table*}[t]
\centering
\caption{Performance on \texttt{ALFWorld} and \texttt{WebShop}. Following GiGPO, for \texttt{ALFWorld} environment we report the average success rate (\%) for each subtask along with the overall result. For \texttt{WebShop} environment, we report the average score and the average success rate (\%). Prompting/PPO/RLOO/GRPO/GiGPO rows are quoted from \citet{feng2026group}; ProxMO from \citet{fang2026proximity}.}
\label{tab:main}
\resizebox{\textwidth}{!}{
\begin{tabular}{llccccccc|cc}
\toprule
\multirow{2}{*}{Type} & \multirow{2}{*}{Method} & \multicolumn{7}{c|}{\textbf{ALFWorld}} & \multicolumn{2}{c}{\textbf{WebShop}} \\
 & & Pick & Look & Clean & Heat & Cool & Pick2 & All & Score & Succ.\\
\midrule
\multicolumn{10}{l}{\textit{Closed-Source Model}} \\
Prompting& GPT-4o & 75.3 & 60.8 & 31.2 & 56.7 & 21.6 & 49.8 & 48.0& 31.8 & 23.7\\
Prompting& Gemini-2.5-Pro & 92.8 & 63.3 & 62.1 & 69.0 & 26.6 & 58.7 & 60.3& 42.5 & 35.9\\
\midrule
\multicolumn{10}{l}{\textit{Qwen2.5-1.5B-Instruct}} \\
Prompting& Qwen2.5 & 5.9 & 5.5 & 3.3 & 9.7 & 4.2 & 0.0 & 4.1 & 23.1 & 5.2\\
Prompting& ReAct & 17.4 & 20.5 & 15.7 & 6.2 & 7.7 & 2.0 & 12.8& 40.1& 11.3\\
Prompting& Reflexion & 35.3 & 22.2 & 21.7 & 13.6 & 19.4 & 3.7 & 21.8 & 55.8& 21.9\\
RL Training& PPO (with critic) & 64.8\textsubscript{(3.5)} & 40.5\textsubscript{(6.9)} & 57.1\textsubscript{(4.9)} & 60.6\textsubscript{(6.6)} & 46.4\textsubscript{(4.0)} & 47.4\textsubscript{(1.9)} & 54.4\textsubscript{(3.1)}& 73.8\textsubscript{(3.0)} & 51.5\textsubscript{(2.9)} \\
RL Training& RLOO & 88.3\textsubscript{(3.0)} & 52.8\textsubscript{(8.6)} & 71.0\textsubscript{(5.9)} & 62.8\textsubscript{(8.7)} & 66.4\textsubscript{(5.5)} & 56.9\textsubscript{(4.7)} & 69.7\textsubscript{(2.5)}& 73.9\textsubscript{(5.6)}& 52.1\textsubscript{(6.7)}\\
RL Training& GRPO & 85.3\textsubscript{(1.5)} & 53.7\textsubscript{(8.0)} & 84.5\textsubscript{(6.8)} & 78.2\textsubscript{(7.9)} & 59.7\textsubscript{(5.0)} & 53.5\textsubscript{(5.6)} & 72.8\textsubscript{(3.6)}& 75.8\textsubscript{(3.5)} & 56.8\textsubscript{(3.8)}\\

RL Training& {\methodname{}} & 94.4\textsubscript{(5.9)} & 67.5\textsubscript{(4.6)} & 94.8\textsubscript{(3.8)} & 94.4\textsubscript{(7.8)} & 79.8\textsubscript{(4.7)} & 76.4\textsubscript{(5.4)} & 86.7\textsubscript{(1.7)} & {83.1}\textsubscript{(1.6)}& {65.0}\textsubscript{(3.2)}\\
RL Training& {ProxMO} & 94.3\textsubscript{(--)} & 92.9\textsubscript{(--)} & 89.3\textsubscript{(--)} & 92.2\textsubscript{(--)} & \textbf{89.5}\textsubscript{(--)} & 87.0\textsubscript{(--)} & 90.6\textsubscript{(--)} & {85.3}\textsubscript{(--)}& {67.1}\textsubscript{(--)}\\
\rowcolor{gray!15}RL Training& \textbf{OURS (PPO)} & \textbf{99.5}\textsubscript{(0.7)} & \textbf{95.5}\textsubscript{(6.4)} & \textbf{99.7}\textsubscript{(0.4)} & \textbf{97.1}\textsubscript{(2.9)} & {87.4}\textsubscript{(1.6)} & \textbf{83.9}\textsubscript{(3.7)} & \textbf{94.0}\textsubscript{(1.2)} & \textbf{89.2}\textsubscript{(1.1)} & \textbf{78.3}\textsubscript{(1.2)}\\
\bottomrule
\end{tabular} 
}
\end{table*}

\begin{table*}[t]
\centering
\caption{Exploration efficiency on \texttt{ALFWorld} and \texttt{WebShop}. We compare our method with GiGPO using three metrics: mean episode length, invalid action rate, and repetitive action rate. Both GiGPO and our method are run by us under the same setup. Lower values indicate more productive exploration.}
\label{tab:exploration}
\resizebox{\textwidth}{!}{
\begin{tabular}{lccc|ccc}
\toprule
\multirow{2}{*}{Method} & \multicolumn{3}{c|}{\textbf{ALFWorld}} & \multicolumn{3}{c}{\textbf{WebShop}} \\
 & Mean Episode Length &  Invalid Action Rate & Repetitive Action Rate & Mean Episode Length & Invalid Action Rate & Repetitive Action Rate \\
\midrule
{GiGPO} & 32.27 & 4.08 & 45.31 & 5.53 & 0.14 & 43.25 \\
{Ours} & \textbf{22.02} & \textbf{2.79} & \textbf{11.53} & \textbf{5.18} &  \textbf{0.09} & \textbf{40.83}\\
\bottomrule
\end{tabular}
}
\end{table*}

In this section, we report the experimental results of our proposed approach of productive state-space exploration. We evaluate our method against several recent baselines across two interactive agent benchmarks: \texttt{ALFWorld}~\citep{shridhar2020alfworld}  and \texttt{WebShop}~\cite{yao2022webshop}. With this, we aim to demonstrate two main benefits of our proposed approach. First, the performance improvements due to the intrinsic reward signals. Second, the exploration behavior and its usefulness in overcoming the repeated actions and looping behavior of RL trained LLM agents.

\subsection{Baselines}
\label{sec:baselines}
For the performance comparison, we benchmark against several state-of-the-art recent baseline models. These include prompting frameworks and those trained with RL. The prompting baselines are Gemini-2.5~\cite{team2023gemini}, GPT-4o~\cite{achiam2023gpt}, Qwen-2.5~\citep{qwen2025qwen25technicalreport}, ReAct~\citep{yao2022react} and Reflexion~\citep{shinn2023reflexion} frameworks. The RL trained baselines include PPO~\cite{schulman2017proximal}, RLOO~\cite{kool2019buy,ahmadian2024back}, GRPO~\cite{shao2024deepseekmath}, GiGPO~\cite{feng2026group}, ProxMO~\cite{fang2026proximity}. The most recent and relevant baselines for our approach are GiGPO and ProxMO, since we aim to address challenges in multi-turn LLM agent training with RL. We do not re-run these baselines, all Table~\ref{tab:main} numbers except our method are quoted from their original publications. For the additional analyses (Table~\ref{tab:ood_main}, Sections~\ref{sec:experimental_evaluation}--\ref{sec:ablation} and Appendix~\ref{sec:supp_experiments}), we re-train GiGPO with the latest official \emph{verl-agent} implementation, without the intrinsic penalty, and evaluate it under exactly the same protocol as our agent.

\subsection{Training Details:}
We train Qwen-2.5-1.5B-Instruct model with the proposed state-space exploration rewards along with Proximal Policy Optimization (PPO) based RL training. We run all our experiments on top of \emph{verl-agent} framework released by GiGPO~\cite{feng2026group}. We train for 300 steps with all other hyperparameters kept same as in the framework. We set the maximum steps at 50 for \texttt{ALFWorld} and 15 for \texttt{WebShop} as with GiGPO. The strength of state-space exploration penalty, $\gamma$, as defined in Equation~\ref{eq:main}, is set $\frac{1}{40}$ of the final success reward. Full reproducibility details ($\gamma$ selection, reward normalization, splits) are in Appendix~\ref{sec:reproducibility}.

\begin{figure*}[t]
\centering
  \includegraphics[width=\linewidth]{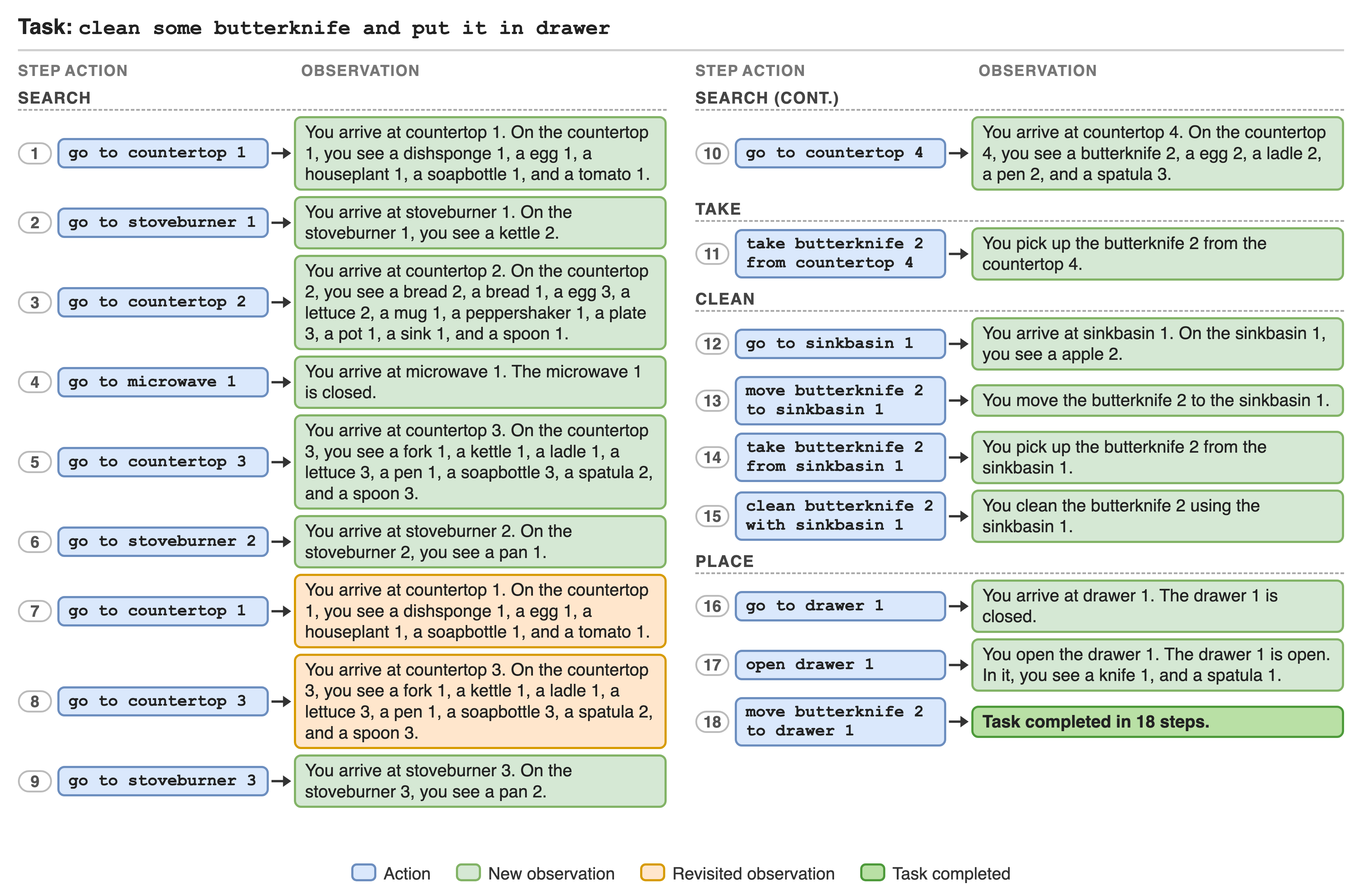}
  \caption{A successful \texttt{ALFWorld} episode of the agent trained with our intrinsic reward, shown in two halves. The agent searches eight locations, revisiting only two, then takes, cleans and places the butterknife in 18 steps without a single invalid action.}
  \label{fig:qualitative}
\end{figure*}
\vspace{-0.4cm}

\subsection{Results}
Table~\ref{tab:main} compares our proposed method (labelled OURS) against existing prompting-based and reinforcement learning (RL) training approaches on the \texttt{ALFWorld} and \texttt{WebShop} benchmarks. Following prior work, we report the average success rate for each \texttt{ALFWorld} subtask together with the overall success rate. For \texttt{WebShop}, we report both the average task score and the average success rate.

We primarily compare our approach against recent RL optimization methods, namely GiGPO (Group-in-Group Policy Optimization) and ProxMO~\cite{fang2026proximity}, which represent strong baselines for LLM agent training. Compared to GiGPO, our method achieves consistent gains across all \texttt{ALFWorld} subtasks, improving the overall success rate from $86.7\%$ to $94.0\%$. The improvements are particularly pronounced on challenging long-horizon tasks such as \textit{Cool} ($+7.6\%$) and \textit{Pick2} ($+7.5\%$), indicating stronger sequential reasoning and more effective action planning. On \texttt{WebShop}, our intrinsic-rewards approach improves the average task score from $83.1$ to $89.2$ and substantially increases the success rate from $65.0\%$ to $78.3\%$, demonstrating superior capability in complex web-based interaction tasks.

Compared to ProxMO, a recent proximity-based multi-turn optimization framework designed for practical credit assignment in LLM agent training, our approach still achieves better overall performance across both benchmarks. Specifically, it improves the overall \texttt{ALFWorld} success rate from $90.6\%$ to $94.0\%$, while increasing the \texttt{WebShop} success rate from $67.1\%$ to $78.3\%$. The improvements on \texttt{ALFWorld} are especially significant, suggesting that our method generalizes more effectively in environments requiring long-horizon exploration, recovery from intermediate mistakes, and consistent multi-turn decision making.

\textbf{Generalization to unseen scenes.} To test whether these gains carry over to environments not seen during training, we evaluate our agent and the re-trained GiGPO agent (Section~\ref{sec:baselines}) on the \texttt{ALFWorld} out-of-distribution split, which contains unseen scenes, and on the in-distribution split under the same protocol (128 episodes each, same games for both agents; Appendix~\ref{sec:supp_experiments}). Since these are single evaluations of one checkpoint per method, the seen-scene numbers differ slightly from Table~\ref{tab:main}, where the GiGPO result is quoted from its original paper. Table~\ref{tab:ood_main} shows that our agent generalizes markedly better. Its success rate is nearly unchanged on unseen scenes, whereas GiGPO's drops noticeably, so the gap between the two agents roughly doubles, and our agent leads on every task type.

\begin{table}[t]
\centering
\small
\caption{Success rate (\%) on seen and unseen \texttt{ALFWorld} scenes for our agent and the re-trained GiGPO.}
\label{tab:ood_main}
\resizebox{\columnwidth}{!}{
\begin{tabular}{llccccccc}
\toprule
Split & Method & Pick & Look & Clean & Heat & Cool & Pick2 & All \\
\midrule
\multirow{2}{*}{Seen} & GiGPO & 92.3 & 84.6 & 92.0 & \textbf{100.0} & \textbf{81.2} & 83.3 & 89.1 \\
 & Ours & \textbf{100.0} & \textbf{100.0} & \textbf{100.0} & 90.9 & \textbf{81.2} & \textbf{91.7} & \textbf{95.3} \\
\midrule
\multirow{2}{*}{Unseen} & GiGPO & 82.6 & 91.3 & 79.2 & 73.7 & 87.0 & 81.2 & 82.8 \\
 & Ours & \textbf{95.7} & \textbf{100.0} & \textbf{91.7} & \textbf{89.5} & \textbf{100.0} & \textbf{87.5} & \textbf{94.5} \\
\bottomrule
\end{tabular}
}
\end{table}

\subsection{Intrinsic Productive Exploration}
\label{sec:experimental_evaluation}

To look beyond overall success rates, we evaluate the behavioral dynamics of our agent against the competitive baseline GiGPO~\citep{feng2026group}. This analysis focuses on \textit{intrinsic productive exploration} - the agent's capacity to autonomously discover optimal state-action trajectories through targeted, purposeful environment interactions, rather than wasting computational steps on redundant or unproductive behaviors. We define three behavioral diagnostics across the \texttt{ALFWorld} and \texttt{WebShop} benchmarks to quantify the efficiency and productivity of the agent's intrinsic exploration strategy:

\begin{itemize}[leftmargin=*]
    \item \textbf{Mean Episode Length:} The average number of sequential steps to complete a task instance.
    \item \textbf{Invalid Action Rate:} The frequency of actions that fail (violate action preconditions). It is detected as described in Appendix~\ref{sec:reproducibility}.
    \item \textbf{Repetitive Action Rate:} The frequency of steps spent executing identical actions. We consider an action repetitive when either the action or the resulting observation is repeated in the trajectory.
\end{itemize}

The empirical results detailing these behavioral metrics are shown in Table~\ref{tab:exploration}. In \texttt{ALFWorld}, our method compresses the {Mean Episode Length} from $\approx 32$ steps down to $\approx 22$ relative to GiGPO. Similar trend persists in the \texttt{WebShop}, where the gains are not as profound as in \texttt{ALFWorld}, however help in score improvement compared to GiGPO. More broadly, the results suggest:
\begin{itemize}[leftmargin=*]
    \item \textbf{Learning of Environment Dynamics:} The improvement in Mean Episode Length and Invalid Action Rate shows that the agent better understands the environment dynamics and, in particular, rarely repeats actions that have already failed.
    \item \textbf{Loop Resolution and Novelty Seeking:} The most supporting evidence of productive exploration is revealed by the \textit{Repetitive Action Rate} in \texttt{ALFWorld}. While GiGPO spends $45.31\%$ of its execution trajectory repeating identical commands, our approach drops this rate dramatically to $11.53\%$. This absolute reduction of $33.78\%$ demonstrates that our intrinsic penalty inherently disincentivizes stale, repeated states. Instead, it forces the agent to continuously break out of loops and seek novel, progressive environmental interactions.
\end{itemize}

Both the above inferences can also be seen in qualitative action trajectories as shown in Figure~\ref{fig:qualitative}. The contrast in the decrease in the number of invalid actions relative to Figure~\ref{fig:intro} suggests our approach of penalizing actions which repeat states is indeed helpful in reducing the redundancy of the agent's trajectory.

To examine information gathering more closely, we record every step of our agent and of the re-trained GiGPO agent (Section~\ref{sec:baselines}). Our agent searches the partially observable environment more widely while wasting fewer steps: it visits more distinct locations, revisits fewer of them, obtains new information from a larger share of its information-gathering actions, and finds the target object more often. The two agents also fail differently. GiGPO's failed episodes are dominated by repeated observations, whereas ours keep exploring new states until the step budget runs out. These advantages hold on both seen and unseen scenes. Per action, both agents satisfy preconditions for nearly all of their actions; our agent's advantage lies in avoiding repeated failures rather than in avoiding a first one (Appendix~\ref{sec:supp_experiments}).

\subsection{Effect of the Intrinsic Penalty on Exploration}
\label{sec:ablation}

To isolate the effect of the penalty, we train two PPO agents from the same base model with identical settings, differing only in the penalty weight ($\gamma=0$ vs.\ $\gamma=\frac{1}{40}$ of the success reward), and evaluate them after 100 training steps (details in Appendix~\ref{sec:supp_experiments}). We measure exploration by the duplicate-observation rate (steps whose observation repeats an earlier one), the location revisit rate (\texttt{go to} actions to an already visited location), the new-information rate (information-gathering actions \texttt{go to}, \texttt{open} and \texttt{examine} that reveal a new observation), and whether the target object is ever observed. A precondition violation is an executed action outside the environment's admissible set; we report violations after the first one in an episode and the recovery rate, i.e.\ how often the action following a violation is admissible.

\textit{The penalty broadens exploration during training.} Starting from identical behavior, the penalized agent reduces repeated observations much faster and covers a substantially larger part of the state space (Figure~\ref{fig:gamma_ablation}). This broader exploration precedes its gain in success, which emerges after the first 50 steps.

\textit{The learned policy explores productively at test time.} It repeats observations and revisits locations far less often, gains new information from more of its actions, and finds the target object more often, equally on seen and unseen scenes (Table~\ref{tab:ablation}).

\textit{Better exploration leads to better use of feedback and higher success.} Since repeated failures reveal nothing new, the penalized agent makes far fewer repeated precondition violations, recovers more often, satisfies preconditions more often, and succeeds considerably more often on both splits. Absolute success is lower than in Table~\ref{tab:main} because of the shorter training.

\begin{figure}[t]
\centering
\includegraphics[width=\columnwidth]{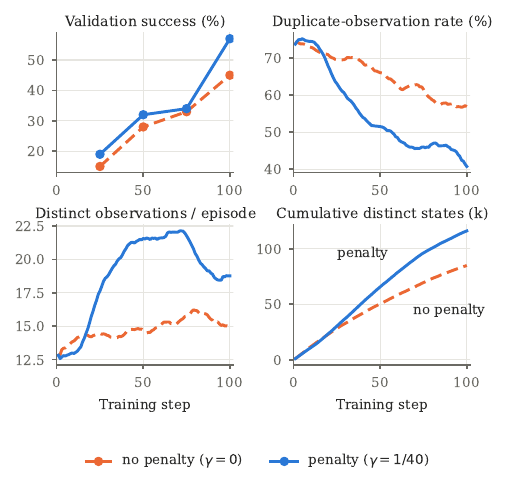}
\caption{Controlled comparison of our model with and without the intrinsic penalty during the first 100 training steps (identical settings). Top: validation success and duplicate-observation rate in the training rollouts. Bottom: distinct observations per episode in the training rollouts and the cumulative number of distinct (game, observation) states visited during training. Rollout curves are 10-step moving averages.}
\label{fig:gamma_ablation}
\end{figure}

\begin{table}[t]
\centering
\small
\caption{Effect of the intrinsic penalty: PPO with $\gamma=0$ vs.\ ours (identical settings, 100 training steps), same 128 episodes per split.}
\label{tab:ablation}
\resizebox{\columnwidth}{!}{
\begin{tabular}{lcccc}
\toprule
 & \multicolumn{2}{c}{Seen scenes} & \multicolumn{2}{c}{Unseen scenes} \\
\cmidrule(lr){2-3}\cmidrule(lr){4-5}
Metric & $\gamma=0$ & Ours & $\gamma=0$ & Ours \\
\midrule
\multicolumn{5}{l}{\textit{Exploration}} \\
\quad Duplicate-observation rate & 59.3\% & \textbf{37.0\%} & 61.8\% & \textbf{36.6\%} \\
\quad Location revisit rate & 68.6\% & \textbf{41.3\%} & 68.7\% & \textbf{42.3\%} \\
\quad New-information rate of info.\ actions & 43.5\% & \textbf{67.4\%} & 40.4\% & \textbf{67.1\%} \\
\quad Distinct locations visited & 4.68 & \textbf{6.17} & 4.87 & \textbf{6.88} \\
\quad Target object found & 76.6\% & \textbf{87.5\%} & 64.8\% & \textbf{85.9\%} \\
\midrule
\multicolumn{5}{l}{\textit{Preconditions}} \\
\quad Precondition satisfied & 68.0\% & \textbf{82.4\%} & 67.8\% & \textbf{81.9\%} \\
\quad Violations after the first, per episode & 9.87 & \textbf{4.65} & 10.12 & \textbf{5.09} \\
\quad Recovery after a violation & 47.2\% & \textbf{59.9\%} & 50.8\% & \textbf{64.2\%} \\
\bottomrule
\end{tabular}
}
\end{table}

\section{Conclusion}
In this paper, we addressed the tendency of RL-trained LLM agent policies to repeat invalid actions and fall into action loops in multi-turn, long-horizon environments.
We introduced a conditional state-space exploration framework driven by an information-aware intrinsic reward.
Instead of directly maximizing unconstrained state entropy, which risks driving the agent toward task-irrelevant trajectories, our approach reframes exploration by penalizing actions that produce redundant state information. Our experimental evaluation demonstrates that this framework achieves substantial performance gains and significantly mitigates action repetition and circular reasoning patterns.

\section*{Limitations}

Our experiments are conducted primarily on smaller reasoning models. Because exploration strategies are inherently influenced by a model's overall capacity and pre-trained priors, larger instruction-tuned models may exhibit fewer failures than those documented in this study. Growth in model scale often mitigates baseline reasoning bottlenecks; nevertheless, optimizing smaller models remains highly valuable. Enhancing the capabilities of these smaller architectures enables more practical, targeted applications, particularly in hardware-constrained environments where deployment efficiency is critical.

Additionally, this work focuses specifically on text-based environments within ALFWorld and WebShop, which feature relatively small, well-defined action spaces that lend themselves well to set representations. Environments with larger, continuous, or stochastic action spaces may exhibit vastly different dynamics under this set-based state representation, potentially limiting direct generalizability.

Regardless of these limitations, our approach provides a structurally sound framework for agent reasoning and state representation in multi-turn settings. Ultimately, the insights gained from our set-based representation and reward design establish a foundation for future scaling and adaptation across more complex, open-ended environments.

\section*{Acknowledgments}
This research was supported by fundings from the Hong Kong RGC General Research Fund (152228/23E, 162161/24E, 162116/25E, 162180/25E), National Natural Science Foundation of China (NSFC) Key Program (No.62532005), Collaborative Research Fund (No. C1042-23GF, No. 5097-25G), NSFC/RGC Collaborative Research Scheme (Grant No. 62461160332 \& CRS\_HKUST602/24), Research Impact Fund (No. R5011-23), and Areas of Excellence Scheme (AoE/E-601/22-R).

\bibliography{custom}

@inproceedings{carta2023grounding,
  title={Grounding large language models in interactive environments with online reinforcement learning},
  author={Carta, Thomas and Romac, Cl{\'e}ment and Wolf, Thomas and Lamprier, Sylvain and Sigaud, Olivier and Oudeyer, Pierre-Yves},
  booktitle={International conference on machine learning},
  pages={3676--3713},
  year={2023},
  organization={PMLR}
}

@article{guan2026evaluating,
  title={Evaluating llm-based agents for multi-turn conversations: A survey},
  author={Guan, Shengyue and Wang, Jindong and Bian, Jiang and Zhu, Bin and Lou, Jian-Guang and Xiong, Haoyi},
  journal={ACM Transactions on Intelligent Systems and Technology},
  volume={17},
  number={4},
  pages={1--40},
  year={2026},
  publisher={ACM New York, NY}
}

@article{shani2024multi,
  title={Multi-turn reinforcement learning with preference human feedback},
  author={Shani, Lior and Rosenberg, Aviv and Cassel, Asaf and Lang, Oran and Calandriello, Daniele and Zipori, Avital and Noga, Hila and Keller, Orgad and Piot, Bilal and Szpektor, Idan and others},
  journal={Advances in Neural Information Processing Systems},
  volume={37},
  pages={118953--118993},
  year={2024}
}

@article{liu2025reinforcement,
  title={Reinforcement learning meets large language models: A survey of advancements and applications across the llm lifecycle},
  author={Liu, Keliang and Yang, Dingkang and Qian, Ziyun and Yin, Weijie and Wang, Yuchi and Li, Hongsheng and Liu, Jun and Zhai, Peng and Liu, Yang and Zhang, Lihua},
  journal={arXiv preprint arXiv:2509.16679},
  year={2025}
}

@article{zhang2025survey,
  title={A survey of reinforcement learning for large reasoning models},
  author={Zhang, Kaiyan and Zuo, Yuxin and He, Bingxiang and Sun, Youbang and Liu, Runze and Jiang, Che and Fan, Yuchen and Tian, Kai and Jia, Guoli and Li, Pengfei and others},
  journal={arXiv preprint arXiv:2509.08827},
  year={2025}
}

@article{xi2025agentgym,
  title={Agentgym-rl: Training llm agents for long-horizon decision making through multi-turn reinforcement learning},
  author={Xi, Zhiheng and Huang, Jixuan and Liao, Chenyang and Huang, Baodai and Guo, Honglin and Liu, Jiaqi and Zheng, Rui and Ye, Junjie and Zhang, Jiazheng and Chen, Wenxiang and others},
  journal={arXiv preprint arXiv:2509.08755},
  year={2025}
}

@article{kalyan2025reinforcement,
  title={Reinforcement Learning for Long-Horizon Multi-Turn Search Agents},
  author={Kalyan, Vivek and Andrews, Martin},
  journal={arXiv preprint arXiv:2510.24126},
  year={2025}
}

@article{chentanez2004intrinsically,
  title={Intrinsically motivated reinforcement learning},
  author={Chentanez, Nuttapong and Barto, Andrew and Singh, Satinder},
  journal={Advances in neural information processing systems},
  volume={17},
  year={2004}
}

@article{aubret2019survey,
  title={A survey on intrinsic motivation in reinforcement learning},
  author={Aubret, Arthur and Matignon, Laetitia and Hassas, Salima},
  journal={arXiv preprint arXiv:1908.06976},
  year={2019}
}

@article{gao2025navigate,
  title={Navigate the unknown: Enhancing llm reasoning with intrinsic motivation guided exploration},
  author={Gao, Jingtong and Pan, Ling and Wang, Yejing and Zhong, Rui and Lu, Chi and Wang, Maolin and Cai, Qingpeng and Jiang, Peng and Zhao, Xiangyu},
  journal={arXiv preprint arXiv:2505.17621},
  year={2025}
}

@inproceedings{tennant2025moral,
  title={Moral Alignment for {LLM} Agents},
  author={Tennant, Elizaveta and Hailes, Stephen and Musolesi, Mirco},
  booktitle={International Conference on Learning Representations},
  year={2025}
}

@article{shinn2023reflexion,
  title={Reflexion: Language agents with verbal reinforcement learning},
  author={Shinn, Noah and Cassano, Federico and Gopinath, Ashwin and Narasimhan, Karthik and Yao, Shunyu},
  journal={Advances in neural information processing systems},
  volume={36},
  pages={8634--8652},
  year={2023}
}

@article{schulman2017proximal,
  title={Proximal policy optimization algorithms},
  author={Schulman, John and Wolski, Filip and Dhariwal, Prafulla and Radford, Alec and Klimov, Oleg},
  journal={arXiv preprint arXiv:1707.06347},
  year={2017}
}

@article{shao2024deepseekmath,
  title={Deepseekmath: Pushing the limits of mathematical reasoning in open language models},
  author={Shao, Zhihong and Wang, Peiyi and Zhu, Qihao and Xu, Runxin and Song, Junxiao and Bi, Xiao and Zhang, Haowei and Zhang, Mingchuan and Li, YK and Wu, Yang and others},
  journal={arXiv preprint arXiv:2402.03300},
  year={2024}
}

@article{kool2019buy,
  title={Buy 4 reinforce samples, get a baseline for free!},
  author={Kool, Wouter and van Hoof, Herke and Welling, Max},
  year={2019}
}

@inproceedings{ahmadian2024back,
  title={Back to basics: Revisiting REINFORCE-style optimization for learning from human feedback in LLMs},
  author={Ahmadian, Arash and Cremer, Chris and Gall{\'e}, Matthias and Fadaee, Marzieh and Kreutzer, Julia and Pietquin, Olivier and {\"U}st{\"u}n, Ahmet and Hooker, Sara},
  booktitle={Proceedings of the 62nd Annual Meeting of the Association for Computational Linguistics (Volume 1: Long Papers)},
  pages={12248--12267},
  year={2024}
}

@article{feng2024agile,
  title={Agile: A novel reinforcement learning framework of llm agents},
  author={Feng, Peiyuan and He, Yichen and Huang, Guanhua and Lin, Yuan and Zhang, Hanchong and Zhang, Yuchen and Li, Hang},
  journal={Advances in Neural Information Processing Systems},
  volume={37},
  pages={5244--5284},
  year={2024}
}

@inproceedings{xiang2024retrospex,
  title={Retrospex: Language agent meets offline reinforcement learning critic},
  author={Xiang, Yufei and Shen, Yiqun and Zhang, Yeqin and Cam-Tu, Nguyen},
  booktitle={Proceedings of the 2024 Conference on Empirical Methods in Natural Language Processing},
  pages={4650--4666},
  year={2024}
}

@inproceedings{zhou2024archer,
author = {Zhou, Yifei and Zanette, Andrea},
title = {ArCHer: training language model agents via hierarchical multi-turn RL},
year = {2024},
publisher = {JMLR.org},
booktitle = {Proceedings of the 41st International Conference on Machine Learning},
articleno = {2574},
numpages = {32},
location = {Vienna, Austria},
series = {ICML'24}
}

@inproceedings{ji2025treesearch,
  title={Tree Search for {LLM} Agent Reinforcement Learning},
  author={Ji, Yuxiang and Ma, Ziyu and Wang, Yong and Chen, Guanhua and Chu, Xiangxiang and Wu, Liaoni},
  booktitle={International Conference on Learning Representations},
  year={2026}
}

@article{xu2025epo,
  title={Epo: Entropy-regularized policy optimization for llm agents reinforcement learning},
  author={Xu, Wujiang and Zhao, Wentian and Wang, Zhenting and Li, Yu-Jhe and Jin, Can and Jin, Mingyu and Mei, Kai and Wan, Kun and Metaxas, Dimitris N},
  journal={arXiv preprint arXiv:2509.22576},
  year={2025}
}

@article{wang2025ragen,
  title={Ragen: Understanding self-evolution in llm agents via multi-turn reinforcement learning},
  author={Wang, Zihan and Wang, Kangrui and Wang, Qineng and Zhang, Pingyue and Li, Linjie and Yang, Zhengyuan and Jin, Xing and Yu, Kefan and Nguyen, Minh Nhat and Liu, Licheng and others},
  journal={arXiv preprint arXiv:2504.20073},
  year={2025}
}

@inproceedings{pathak2017curiosity,
  title={Curiosity-driven exploration by self-supervised prediction},
  author={Pathak, Deepak and Agrawal, Pulkit and Efros, Alexei A and Darrell, Trevor},
  booktitle={International conference on machine learning},
  pages={2778--2787},
  year={2017},
  organization={PMLR}
}

@inproceedings{burda2018large,
  Author = {Burda, Yuri and
  Edwards, Harri and Pathak, Deepak and
  Storkey, Amos and Darrell, Trevor and
  Efros, Alexei A.},
  Title = {Large-Scale Study of
  Curiosity-Driven Learning},
  Booktitle = {ICLR},
  Year = {2019}
}

@inproceedings{bellemare2016unifying,
author = {Bellemare, Marc G. and Srinivasan, Sriram and Ostrovski, Georg and Schaul, Tom and Saxton, David and Munos, R\'{e}mi},
title = {Unifying count-based exploration and intrinsic motivation},
year = {2016},
isbn = {9781510838819},
publisher = {Curran Associates Inc.},
address = {Red Hook, NY, USA},
booktitle = {Proceedings of the 30th International Conference on Neural Information Processing Systems},
pages = {1479–1487},
numpages = {9},
location = {Barcelona, Spain},
series = {NIPS'16}
}

@article{kim2023accelerating,
  title={Accelerating reinforcement learning with value-conditional state entropy exploration},
  author={Kim, Dongyoung and Shin, Jinwoo and Abbeel, Pieter and Seo, Younggyo},
  journal={Advances in Neural Information Processing Systems},
  volume={36},
  pages={31811--31830},
  year={2023}
}

@inproceedings{mutti2021task,
  title={Task-agnostic exploration via policy gradient of a non-parametric state entropy estimate},
  author={Mutti, Mirco and Pratissoli, Lorenzo and Restelli, Marcello},
  booktitle={Proceedings of the AAAI conference on artificial intelligence},
  volume={35},
  number={10},
  pages={9028--9036},
  year={2021}
}

@inproceedings{ashlag2026state,
  title={State Entropy Regularization for Robust Reinforcement Learning},
  author={Ashlag, Yonatan and Koren, Uri and Mutti, Mirco and Derman, Esther and Bacon, Pierre-Luc and Mannor, Shie},
  booktitle={Advances in Neural Information Processing Systems},
  year={2025}
}

@inproceedings{
nourzad2026mira,
title={{MIRA}: Memory-Integrated Reinforcement Learning Agent  with Limited {LLM} Guidance},
author={Narjes Nourzad and Carlee Joe-Wong},
booktitle={The Fourteenth International Conference on Learning Representations},
year={2026},
url={https://openreview.net/forum?id=oWagByDNPc}
}

@article{fang2026proximity,
  title={Proximity-Based Multi-Turn Optimization: Practical Credit Assignment for LLM Agent Training},
  author={Fang, Yangyi and Lin, Jiaye and Fu, Xiaoliang and Qin, Cong and Shi, Haolin and Liu, Chang and Zhao, Peilin},
  journal={arXiv preprint arXiv:2602.19225},
  year={2026}
}

@inproceedings{feng2026group,
  title={Group-in-Group Policy Optimization for {LLM} Agent Training},
  author={Feng, Lang and Xue, Zhenghai and Liu, Tingcong and An, Bo},
  booktitle={Advances in Neural Information Processing Systems},
  year={2025}
}

@article{shridhar2020alfworld,
  title={Alfworld: Aligning text and embodied environments for interactive learning},
  author={Shridhar, Mohit and Yuan, Xingdi and C{\^o}t{\'e}, Marc-Alexandre and Bisk, Yonatan and Trischler, Adam and Hausknecht, Matthew},
  journal={arXiv preprint arXiv:2010.03768},
  year={2020}
}

@article{yao2022webshop,
  title={Webshop: Towards scalable real-world web interaction with grounded language agents},
  author={Yao, Shunyu and Chen, Howard and Yang, John and Narasimhan, Karthik},
  journal={Advances in Neural Information Processing Systems},
  volume={35},
  pages={20744--20757},
  year={2022}
}

@misc{qwen2025qwen25technicalreport,
      title={Qwen2.5 Technical Report}, 
      author={Qwen and : and An Yang and Baosong Yang and Beichen Zhang and Binyuan Hui and Bo Zheng and Bowen Yu and Chengyuan Li and Dayiheng Liu and Fei Huang and Haoran Wei and Huan Lin and Jian Yang and Jianhong Tu and Jianwei Zhang and Jianxin Yang and Jiaxi Yang and Jingren Zhou and Junyang Lin and Kai Dang and Keming Lu and Keqin Bao and Kexin Yang and Le Yu and Mei Li and Mingfeng Xue and Pei Zhang and Qin Zhu and Rui Men and Runji Lin and Tianhao Li and Tianyi Tang and Tingyu Xia and Xingzhang Ren and Xuancheng Ren and Yang Fan and Yang Su and Yichang Zhang and Yu Wan and Yuqiong Liu and Zeyu Cui and Zhenru Zhang and Zihan Qiu},
      year={2025},
      eprint={2412.15115},
      archivePrefix={arXiv},
      primaryClass={cs.CL},
      url={https://arxiv.org/abs/2412.15115}, 
}

@inproceedings{yao2022react,
  title={{ReAct}: Synergizing Reasoning and Acting in Language Models},
  author={Yao, Shunyu and Zhao, Jeffrey and Yu, Dian and Du, Nan and Shafran, Izhak and Narasimhan, Karthik and Cao, Yuan},
  booktitle={International Conference on Learning Representations},
  year={2023}
}

@article{achiam2023gpt,
  title={Gpt-4 technical report},
  author={Achiam, Josh and Adler, Steven and Agarwal, Sandhini and Ahmad, Lama and Akkaya, Ilge and Aleman, Florencia Leoni and Almeida, Diogo and Altenschmidt, Janko and Altman, Sam and Anadkat, Shyamal and others},
  journal={arXiv preprint arXiv:2303.08774},
  year={2023}
}

@article{team2023gemini,
  title={Gemini: a family of highly capable multimodal models},
  author={Team, Gemini and Anil, Rohan and Borgeaud, Sebastian and Alayrac, Jean-Baptiste and Yu, Jiahui and Soricut, Radu and Schalkwyk, Johan and Dai, Andrew M and Hauth, Anja and Millican, Katie and others},
  journal={arXiv preprint arXiv:2312.11805},
  year={2023}
}

\clearpage

\appendix
\section{Appendix}
\label{sec:appendix}

\subsection{Datasets:}

\paragraph{ALFWorld.}
\texttt{ALFWorld}~\citep{shridhar2020alfworld} is a text-based interactive environment designed for studying embodied decision making and long-horizon reasoning in language agents. Derived from the ALFRED benchmark, it converts household manipulation tasks into a partially observable textual setting where agents interact through natural language actions and observations. Tasks typically involve navigation, object manipulation, and sequential planning over multiple turns, making the environment particularly suitable for evaluating reinforcement learning methods for LLM agents under sparse rewards and delayed feedback. Concretely, \texttt{ALFWorld} contains 6 task types of various difficulty. The tasks are:
\begin{itemize}
    \item Pick \& Place: agent must find an object of a specified type, pick it up, find the correct location to place it, and place it down there.
    \item Examine in Light: agent must find an object of a specified type, locate and turn on a light source with the object still in hand.
    \item Clean \& Place: agent must find an object of a specified type, pick it up, move to a sink or a basin, turn on the faucet and wash the object, find the correct location to place it, and place it down there.
    \item Heat \& Place: agent must find an object of a specified type, pick it up, move to the microwave, heat the object by turning on the microwave, find the correct location to place it, and place it down there.
    \item Cool \& Place: agent must find an object of a specified type, pick it up, move to the fridge, place the object in the fridge to cool it, find the correct location to place it, and place it down there.
    \item Pick Two \& Place: agent must find an object of a specified type, find the correct location to place it, place it down there, look for another object of a specified type, pick up the object, return to the previous location, place it down there together with the first object.
\end{itemize}

\paragraph{WebShop.}
\texttt{WebShop}~\citep{yao2022webshop} is an e-commerce interaction benchmark that evaluates language agents on realistic online shopping tasks. In this environment, agents must search through large collections of products, inspect item descriptions and reviews, and iteratively refine decisions to satisfy user instructions. Unlike embodied environments, WebShop emphasizes retrieval, preference understanding, and multi-step reasoning in semi-structured web settings. The benchmark is widely used for evaluating the effectiveness of sequential decision-making and reinforcement learning approaches for interactive LLM agents. Concretely, the \texttt{WebShop} dataset contains 1,000,000 products, 670 unique attributes, 842,849 unique options, extracted publicly from Amazon. It presents over 12,000 annotated trajectories and an action space of 8 valid actions.

\subsection{Additional Experiments}
\label{sec:supp_experiments}

This subsection reports additional \texttt{ALFWorld} experiments that analyze (i) whether the trained agent gathers information more effectively at test time, and (ii) whether more of its actions satisfy their preconditions. Sections~\ref{sec:experimental_evaluation} and~\ref{sec:ablation} summarize the main findings; here we give the full results.

\subsubsection{Test-Time Exploration and Preconditions}
\label{sec:supp_gigpo}

\paragraph{Setup.}
We compare our agent with a GiGPO agent that we trained ourselves using the released \emph{verl-agent} code. Both agents play the same 128 evaluation episodes (91 distinct games) from the \texttt{ALFWorld} in-distribution split, with sampling temperature 0.4 and at most 50 steps. We record every step (observation, admissible actions, chosen action and next observation) and compute all metrics below per episode, then average over episodes.

\paragraph{Test-time exploration and information gathering.}
\texttt{ALFWorld} is partially observable: the contents of a receptacle are revealed only when the agent goes to it or opens it. Table~\ref{tab:supp_info} measures how the agents gather information. Our agent covers more of the environment with less effort: it visits more distinct locations per episode while taking fewer steps, revisits locations less often, and a larger share of its information-gathering actions (\texttt{go to}, \texttt{open}, \texttt{examine}) reveals a previously unseen observation. As a result, it finds the target object more often. When both agents find the target, our agent gets there in fewer steps despite searching more locations, i.e.\ it searches more widely and wastes fewer steps on revisits. The failure modes also differ. Failed episodes of both agents exhaust the step budget, but GiGPO's failures are dominated by repetition, whereas our failed episodes keep exploring new states.

\begin{table}[h]
\centering
\small
\caption{Test-time information gathering (per-episode means). $^\dagger$Episodes in which both agents find the target.}
\label{tab:supp_info}
\resizebox{\columnwidth}{!}{
\begin{tabular}{lcc}
\toprule
Metric & GiGPO & Ours \\
\midrule
Success rate & 89.1\% & \textbf{95.3\%} \\
Episode length (steps) & 14.7 & \textbf{12.5} \\
Distinct locations visited & 3.66 & \textbf{5.16} \\
Location revisit rate & 31.5\% & \textbf{18.9\%} \\
New-information rate of info.\ actions & 81.2\% & \textbf{89.2\%} \\
Duplicate-observation rate & 16.7\% & \textbf{8.7\%} \\
Target object found & 91.4\% & \textbf{96.9\%} \\
Steps until target found$^\dagger$ & 4.46 & \textbf{4.19} \\
Locations searched before target$^\dagger$ & 2.27 & 3.19 \\
Duplicate-obs.\ rate in failed episodes & 72\% & \textbf{27\%} \\
\bottomrule
\end{tabular}
}
\end{table}

\paragraph{Precondition satisfaction at test time.}
\texttt{ALFWorld} provides the set of admissible actions in each state, which are exactly the actions whose preconditions are satisfied. We therefore measure, for each action type, how many actions each agent executes and the fraction of them that are admissible (Table~\ref{tab:supp_precond}). On the same episodes, our agent completes the tasks with about $16\%$ fewer actions. The saving comes almost entirely from navigation, as it needs far fewer \texttt{go to} actions, while it uses information-gathering actions such as \texttt{open} and \texttt{examine} more often. Both agents satisfy preconditions for nearly all of their actions, with a small difference in GiGPO's favor overall. The per-action picture is mixed: our agent is more reliable on \texttt{take} and \texttt{examine}, while GiGPO is more reliable on \texttt{put}. Against this baseline, our agent therefore does not show a higher precondition-satisfaction rate per action at test time. Its advantage lies in reaching the goal with fewer actions and in avoiding \emph{repeated} uninformative actions (Table~\ref{tab:supp_info}), which is the behavior the penalty directly targets.

\begin{table}[h]
\centering
\small
\caption{Actions per episode and precondition satisfaction by action type. $^\ddagger$Includes \texttt{close} and \texttt{look}.}
\label{tab:supp_precond}
\resizebox{\columnwidth}{!}{
\begin{tabular}{lcccccc}
\toprule
 & \multicolumn{3}{c}{GiGPO} & \multicolumn{3}{c}{Ours} \\
\cmidrule(lr){2-4}\cmidrule(lr){5-7}
Action & Per ep. & $n$ & Satisfied & Per ep. & $n$ & Satisfied \\
\midrule
\texttt{go to}   & 10.24 & 1311 & 98.6\%  & \textbf{7.26} & 929 & 98.3\%  \\
\texttt{take}    & 1.58  & 202  & 93.6\%  & 1.43 & 183 & 100.0\% \\
\texttt{put}     & 1.38  & 176  & 99.4\%  & 1.40 & 179 & 92.2\%  \\
\texttt{open}    & 0.74  & 95   & 98.9\%  & 1.34 & 171 & 99.4\%  \\
\texttt{examine} & 0.07  & 9    & 66.7\%  & 0.37 & 47  & 80.9\%  \\
\texttt{clean}   & 0.22  & 28   & 100.0\% & 0.24 & 31  & 96.8\%  \\
\texttt{cool}    & 0.13  & 17   & 100.0\% & 0.13 & 17  & 100.0\% \\
\texttt{heat}    & 0.09  & 12   & 100.0\% & 0.09 & 11  & 100.0\% \\
\texttt{use}     & 0.23  & 30   & 100.0\% & 0.11 & 14  & 100.0\% \\
\midrule
All actions$^\ddagger$ & 14.71 & 1883 & 98.0\% & \textbf{12.38} & 1584 & 97.4\% \\
\bottomrule
\end{tabular}
}
\end{table}

\paragraph{Repeated precondition violations.}
In \texttt{ALFWorld}, every action whose preconditions are not satisfied returns the same fixed observation, \texttt{"Nothing happens."} (Appendix~\ref{sec:reproducibility}). The first violation in an episode therefore adds a new observation to the state memory and is not penalized, while every later violation in the same episode repeats this observation and is penalized. The penalty is thus not expected to prevent a first, exploratory violation, but to discourage the agent from violating preconditions \emph{again} after it has already received this feedback. Table~\ref{tab:supp_repeat} measures this behavior on the same 128 episodes, using the following definitions:
\begin{itemize}[leftmargin=*,itemsep=1pt,topsep=2pt]
    \item \textbf{Violation:} an executed action that is not in the admissible set of the state in which it is taken, i.e.\ an action whose preconditions are not satisfied.
    we report violations after the first, which are the violations which receive the intrinsic penalty.
    \emph{Per episode} divides their number by the 128 episodes.
    \item \textbf{Exact repeat of a failed action:} a violation in which the agent executes the very same action string that has already failed earlier in the same episode.
    \item \textbf{Share of all violations:} the percentage of an agent's violations that fall into the given category.
    \item \textbf{Back-to-back violations:} violations that immediately follow another violation.
    \item \textbf{Recovery after a violation:} the percentage of violations whose next action is admissible.
\end{itemize}
Our agent is less likely to violate preconditions again once it has received the failure feedback: it has fewer violations after the first in an episode.
The clearest difference is in exact repeats. About half of GiGPO's violations re-execute an action that has already failed in the same episode, compared with about a quarter of ours. Our agent also recovers more often, with the action following a violation more frequently being admissible. This pattern matches the mechanism of the penalty, which leaves the first violation unpenalized but penalizes repeating it.

\begin{table}[h]
\centering
\small
\caption{Repeated precondition violations (shares of each agent's violations).}
\label{tab:supp_repeat}
\resizebox{\columnwidth}{!}{
\begin{tabular}{lcc}
\toprule
Metric & GiGPO & Ours \\
\midrule
Violations after the first (total) & 26 & \textbf{23} \\
\quad per episode & 0.20 & \textbf{0.18} \\
\quad share of all violations & 70.3\% & \textbf{56.1\%} \\
Exact repeats of a failed action (total) & 19 & \textbf{10} \\
\quad share of all violations & 51.4\% & \textbf{24.4\%} \\
Back-to-back violations & 4 & \textbf{3} \\
Recovery after a violation & 88.9\% & \textbf{92.5\%} \\
Maximum violations in one episode & 8 & \textbf{7} \\
\bottomrule
\end{tabular}
}
\end{table}

\subsubsection{Generalization to Unseen Scenes}
\label{sec:supp_ood}
We repeat all of the above analyses on the \texttt{ALFWorld} out-of-distribution split, which contains scenes not seen during training. Both agents use the same checkpoints and evaluation settings as before, and both play the same 128 episodes (74 distinct games). Table~\ref{tab:supp_ood} summarizes the results, and all metrics follow the definitions given above.

The advantage of our agent grows on unseen scenes. Its success rate barely changes, while GiGPO's drops noticeably, so the gap roughly doubles to about 12 points, and our agent is ahead on every task type.
All information-gathering differences persist, mostly with larger margins.
Its failed episodes keep exploring rather than looping, and it again completes the tasks with fewer actions. On unseen scenes, it also satisfies preconditions slightly more often than GiGPO and makes fewer violations in total.

\begin{table}[h]
\centering
\small
\caption{Results on unseen scenes (out-of-distribution split). $^\dagger$Episodes in which both agents find the target.}
\label{tab:supp_ood}
\resizebox{\columnwidth}{!}{
\begin{tabular}{lcc}
\toprule
Metric & GiGPO & Ours \\
\midrule
\multicolumn{3}{l}{\textit{Information gathering}} \\
\quad Episode length (steps) & 17.4 & \textbf{14.7} \\
\quad Distinct locations visited & 4.22 & \textbf{5.80} \\
\quad Location revisit rate & 35.4\% & \textbf{22.1\%} \\
\quad New-information rate of info.\ actions & 76.7\% & \textbf{88.7\%} \\
\quad Duplicate-observation rate & 20.8\% & \textbf{9.4\%} \\
\quad Target object found & 88.3\% & \textbf{97.7\%} \\
\quad Steps until target found$^\dagger$ & 4.96 & \textbf{4.44} \\
\quad Duplicate-obs.\ rate in failed episodes & 67\% & \textbf{25\%} \\
\midrule
\multicolumn{3}{l}{\textit{Actions and preconditions}} \\
\quad Actions per episode & 17.4 & \textbf{14.6} \\
\quad Precondition satisfied (all actions) & 96.7\% & \textbf{97.3\%} \\
\quad \texttt{take} / \texttt{put} / \texttt{examine} & 94.0 / \textbf{99.4} / 40.0\% & \textbf{98.7} / 94.1 / \textbf{93.0}\% \\
\midrule
\multicolumn{3}{l}{\textit{Repeated violations}} \\
\quad Violations after the first, per episode & 0.38 & \textbf{0.23} \\
\quad \quad share of all violations & 65.8\% & \textbf{58.0\%} \\
\quad Exact repeats of a failed action (total) & 31 & \textbf{19} \\
\quad \quad share of all violations & 42.5\% & \textbf{38.0\%} \\
\quad Back-to-back violations & \textbf{16} & 17 \\
\quad Recovery after a violation & \textbf{78.1\%} & 64.6\% \\
\quad Maximum violations in one episode & 12 & \textbf{7} \\
\bottomrule
\end{tabular}
}
\end{table}

\subsubsection{Controlled Comparison with the Penalty Disabled}
\label{sec:supp_ctrl}
Here we detail the controlled comparison of Section~\ref{sec:ablation}. Both PPO agents start from Qwen2.5-1.5B-Instruct and share all hyperparameters, prompts and the random seed, so they encounter exactly the same training games at every step. We record all training rollouts (100 episodes per step) and evaluate both policies on the same validation episodes every 25 steps, using a single run per setting.

\paragraph{Exploration during training.}
Table~\ref{tab:supp_train} complements Figure~\ref{fig:gamma_ablation} with further training-rollout metrics at selected steps. Besides repeating fewer observations, the penalized agent revisits fewer locations, obtains new information from a larger share of its information-gathering actions, and runs into fewer precondition failures throughout training. The number of distinct observations per episode eventually decreases for the penalized agent because it succeeds more often and its episodes become shorter.

\begin{table}[h]
\centering
\small
\caption{Training rollouts with $\gamma=0$ vs.\ ours, averaged over the 10 steps ending at each step.}
\label{tab:supp_train}
\resizebox{\columnwidth}{!}{
\begin{tabular}{lcccccc}
\toprule
 & \multicolumn{2}{c}{Step 25} & \multicolumn{2}{c}{Step 50} & \multicolumn{2}{c}{Step 100} \\
\cmidrule(lr){2-3}\cmidrule(lr){4-5}\cmidrule(lr){6-7}
Metric & $\gamma=0$ & Ours & $\gamma=0$ & Ours & $\gamma=0$ & Ours \\
\midrule
Duplicate-observation rate & 69.6\% & \textbf{63.2\%} & 66.1\% & \textbf{51.5\%} & 56.9\% & \textbf{40.7\%} \\
Distinct observations / episode & 14.5 & \textbf{17.6} & 14.7 & \textbf{21.6} & 15.1 & \textbf{18.8} \\
Location revisit rate & 67.4\% & \textbf{64.5\%} & 65.6\% & \textbf{58.9\%} & 68.9\% & \textbf{49.8\%} \\
New-information rate of info.\ actions & 33.0\% & \textbf{38.2\%} & 37.8\% & \textbf{48.8\%} & 42.0\% & \textbf{59.7\%} \\
\texttt{"Nothing happens."} rate & 45.7\% & \textbf{40.2\%} & 44.4\% & \textbf{31.1\%} & 31.3\% & \textbf{20.2\%} \\
Cumulative distinct states (k) & 28.3 & \textbf{30.4} & 50.4 & \textbf{65.7} & 84.9 & \textbf{116.0} \\
\midrule
Training success & 11.4\% & 11.4\% & \textbf{21.3\%} & 20.7\% & 41.9\% & \textbf{49.3\%} \\
Episode length (steps) & 47.1 & 47.3 & 44.8 & 44.8 & 38.5 & \textbf{35.2} \\
Validation success & 15\% & \textbf{19\%} & 28\% & \textbf{32\%} & 45\% & \textbf{57\%} \\
\bottomrule
\end{tabular}
}
\end{table}

\paragraph{Test-time behavior of the controlled agents.}
Table~\ref{tab:supp_ctrl_test} extends Table~\ref{tab:ablation} with further metrics, using the same 128 episodes per split as Tables~\ref{tab:supp_info}--\ref{tab:supp_ood}. In addition to the effects discussed in Section~\ref{sec:ablation}, the penalized agent completes episodes with fewer actions, wins far more of the games that only one of the two agents solves, satisfies preconditions more often for almost every action type, and its failed episodes contain less repetition. At this stage of training, both agents still violate preconditions often, so absolute rates are lower than for the fully trained agent in Tables~\ref{tab:supp_info}--\ref{tab:supp_ood}.

\begin{table}[h]
\centering
\small
\caption{Test-time behavior of the controlled agents after 100 training steps.}
\label{tab:supp_ctrl_test}
\resizebox{\columnwidth}{!}{
\begin{tabular}{lcccc}
\toprule
 & \multicolumn{2}{c}{Seen scenes} & \multicolumn{2}{c}{Unseen scenes} \\
\cmidrule(lr){2-3}\cmidrule(lr){4-5}
Metric & $\gamma=0$ & Ours & $\gamma=0$ & Ours \\
\midrule
Success rate & 39.1\% & \textbf{57.8\%} & 34.4\% & \textbf{52.3\%} \\
Episode length (steps) & 37.9 & \textbf{31.0} & 40.5 & \textbf{32.7} \\
Actions per episode & 33.8 & \textbf{30.3} & 34.4 & \textbf{32.1} \\
\midrule
Distinct locations visited & 4.68 & \textbf{6.17} & 4.87 & \textbf{6.88} \\
Location revisit rate & 68.6\% & \textbf{41.3\%} & 68.7\% & \textbf{42.3\%} \\
New-information rate of info.\ actions & 43.5\% & \textbf{67.4\%} & 40.4\% & \textbf{67.1\%} \\
Duplicate-observation rate & 59.3\% & \textbf{37.0\%} & 61.8\% & \textbf{36.6\%} \\
Target object found & 76.6\% & \textbf{87.5\%} & 64.8\% & \textbf{85.9\%} \\
Duplicate-obs.\ rate in failed episodes & 72\% & \textbf{56\%} & 72\% & \textbf{57\%} \\
\midrule
Precondition satisfied (all actions) & 68.0\% & \textbf{82.4\%} & 67.8\% & \textbf{81.9\%} \\
Violations after the first, per episode & 9.87 & \textbf{4.65} & 10.12 & \textbf{5.09} \\
Exact repeats, share of all violations & 62.5\% & \textbf{55.4\%} & 61.7\% & \textbf{54.1\%} \\
Recovery after a violation & 47.2\% & \textbf{59.9\%} & 50.8\% & \textbf{64.2\%} \\
\bottomrule
\end{tabular}
}
\end{table}

\subsubsection{Relation to Theory and Qualitative Example}
\label{sec:supp_theory}

\paragraph{Relation to Proposition~\ref{proposition_coverage}.}
Proposition~\ref{proposition_coverage} bounds the visitation entropy $H(P_L)$ of an episode by its number of duplicate observations $D$. Figure~\ref{fig:entropy_bounds} plots both quantities for all full-length episodes of the controlled agents. Every episode of both agents lies between the two bounds, and across all episodes, entropy decreases almost monotonically with the number of duplicates (rank correlation close to $-1$). The penalized agent produces episodes with clearly fewer duplicates and correspondingly higher visitation entropy, both in failed full-length episodes and on average over all episodes, where the paired confidence interval of the entropy difference excludes zero on both splits. This confirms empirically that minimizing the cumulative penalty, i.e.\ the number of duplicate observations, translates into higher visitation entropy in practice.

\begin{figure}[H]
\centering
\includegraphics[width=\columnwidth]{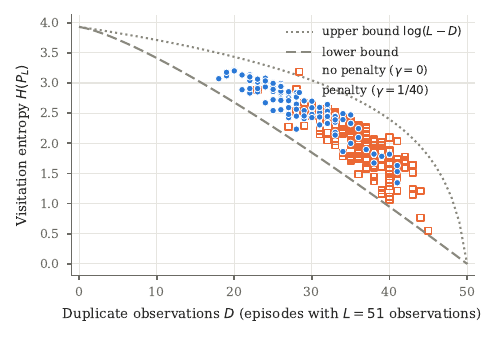}
\caption{Visitation entropy $H(P_L)$ against the number of duplicate observations $D$ for all full-length episodes ($L=51$ observations) of the controlled agents after 100 training steps, seen and unseen splits combined. The dotted and dashed curves are the upper and lower bounds of Proposition~\ref{proposition_coverage}.}
\label{fig:entropy_bounds}
\end{figure}

\paragraph{Qualitative example.}
Figure~\ref{fig:qual_pair} shows the two controlled agents on the same game from the unseen split, taken directly from the evaluation logs. Both agents start identically and pick up the mug. The agent trained without the penalty then repeatedly issues a command for the location it is already at, receives \texttt{"Nothing happens."} over and over, and never cools the mug before the step budget runs out. The penalized agent instead proceeds to the fridge, cools the mug and completes the task in eight steps. This is the looping failure mode of Figure~\ref{fig:intro}, which the penalty removes.

\begin{figure}[H]
\centering
\footnotesize
\setlength{\tabcolsep}{3pt}
\resizebox{\columnwidth}{!}{%
\begin{tabular}{@{}r@{\hspace{4pt}}l@{\hspace{6pt}}l@{}}
\toprule
\multicolumn{3}{@{}l}{\textit{Task: cool some mug and put it in coffeemachine (unseen scene).}} \\
\midrule
\multicolumn{3}{@{}l}{\textbf{Penalty} ($\gamma=\frac{1}{40}$): success in 8 steps} \\
1 & \texttt{go to sinkbasin 1} & arrive; sinkbasin empty \\
2 & \texttt{go to countertop 1} & arrive; see mug 1 \\
3 & \texttt{take mug 1 from countertop 1} & pick up mug 1 \\
4 & \texttt{go to fridge 1} & arrive; fridge closed \\
5--6 & \texttt{cool mug 1 with fridge 1} & cool mug 1 \\
7 & \texttt{go to coffeemachine 1} & arrive \\
8 & \texttt{move mug 1 to coffeemachine 1} & \textbf{task completed} \\
\midrule
\multicolumn{3}{@{}l}{\textbf{No penalty} ($\gamma=0$): failure after 50 steps} \\
1--3 & \multicolumn{2}{@{}l}{same as above (mug 1 picked up)} \\
4 & \texttt{go to sinkbasin 1} & arrive \\
5--7 & \texttt{go to sinkbasin 1} ($\times3$) & \texttt{"Nothing happens."} \\
8--9 & \texttt{go to fridge 1}, \texttt{open fridge 1} & fridge open \\
10 & \texttt{go to sinkbasin 1} & arrive \\
11--21 & \texttt{go to sinkbasin 1} ($\times11$) & \texttt{"Nothing happens."} \\
22--50 & \multicolumn{2}{@{}l}{moves the mug around; 17 more \texttt{"Nothing happens."};} \\
 & \multicolumn{2}{@{}l}{the mug is never cooled; step budget exhausted} \\
\bottomrule
\end{tabular}
}
\caption{The two controlled agents after 100 training steps on the same unseen game (actions and observations from the evaluation logs, abridged). In total, 31 of the 50 steps of the agent without the penalty return \texttt{"Nothing happens."}.}
\label{fig:qual_pair}
\end{figure}

\subsubsection{Additional Qualitative Results}
In Figure~\ref{fig:qualitative_appendix}, we further show some of the example outputs of the LLM agent trained with intrinsic productive exploration rewards framework.

\begin{figure*}[t]
\centering
  \includegraphics[width=\linewidth]{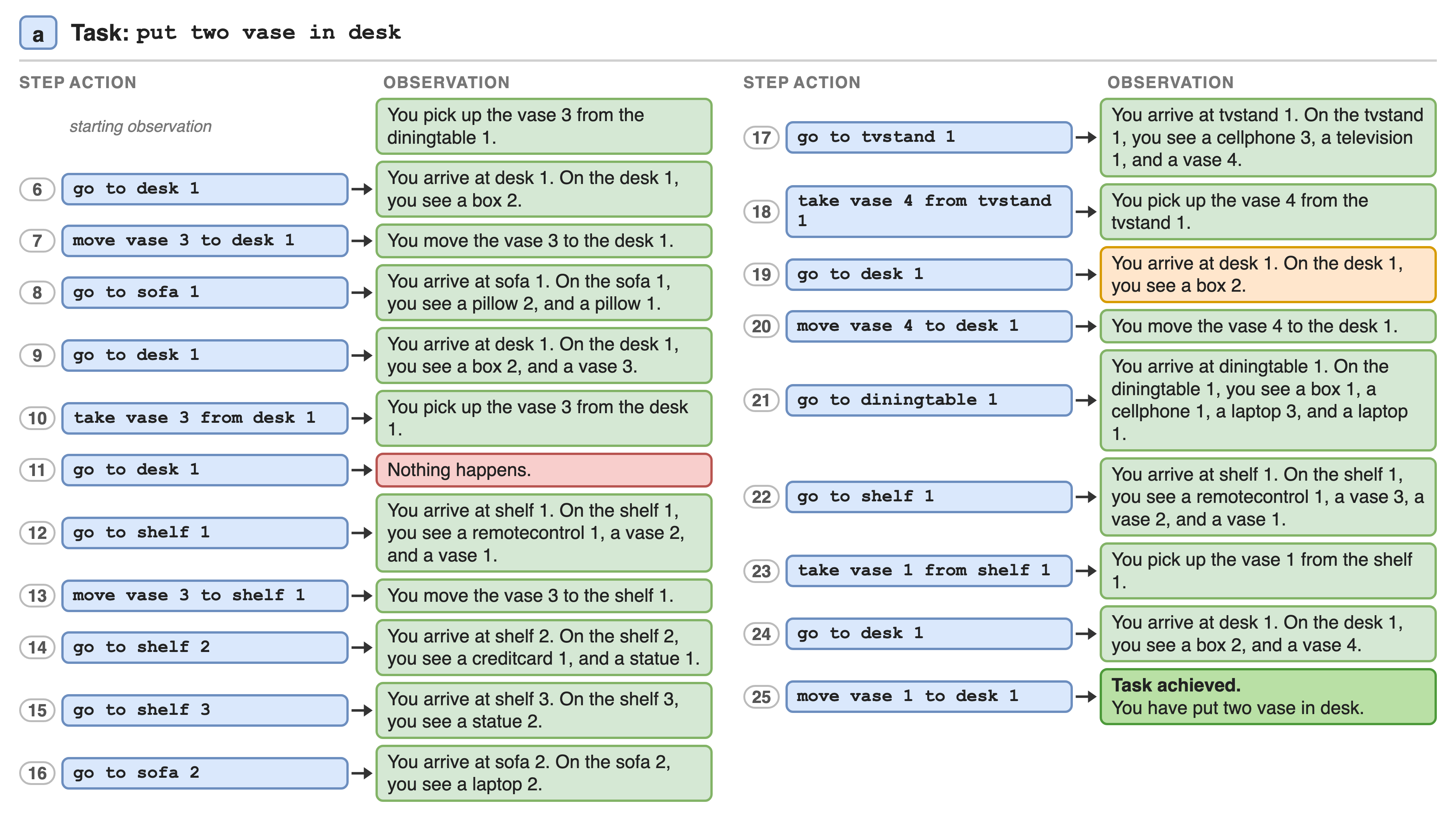}
  \vspace{6pt}
  \includegraphics[width=\linewidth]{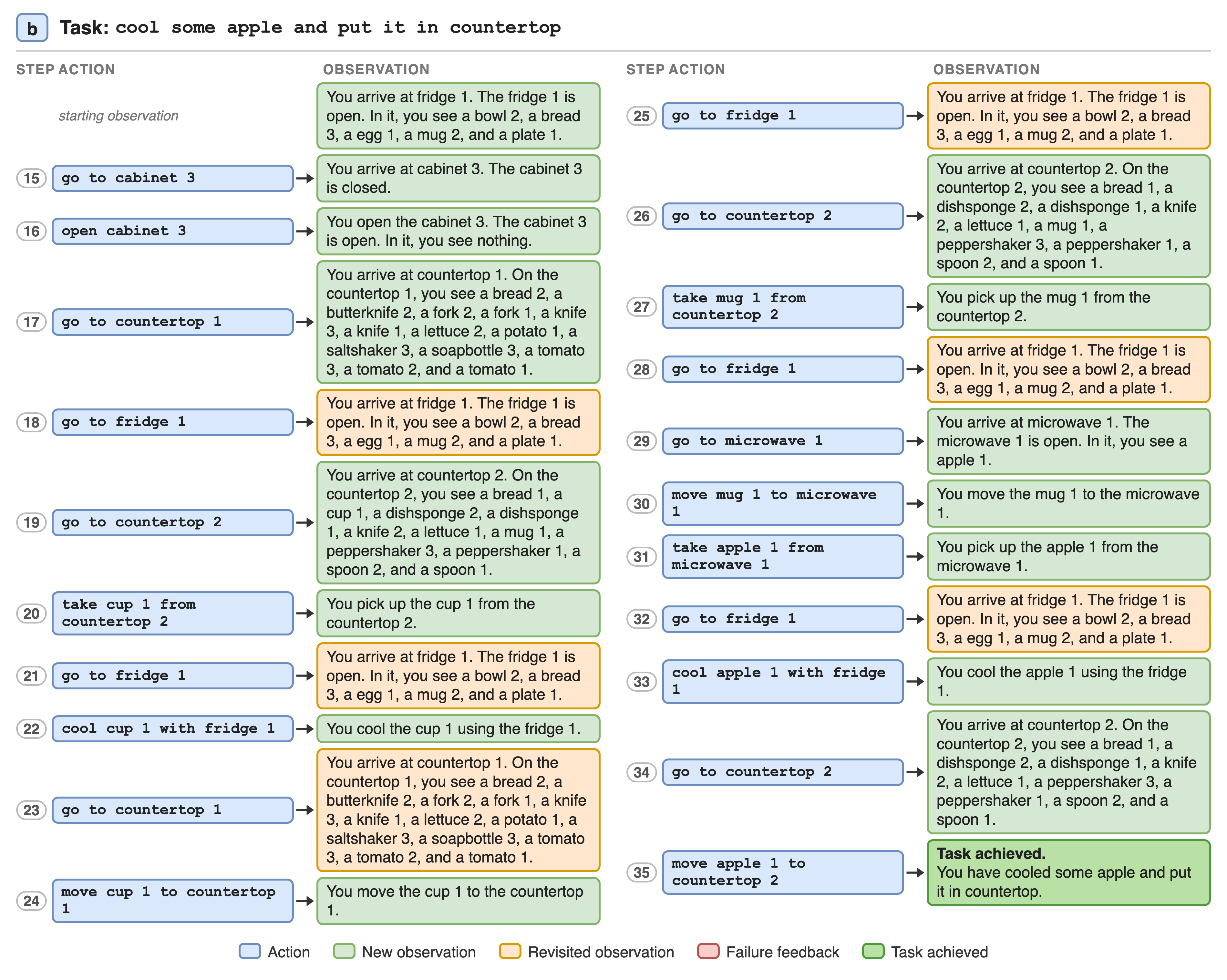}
  \caption {Additional example outputs of the multi-turn interaction of the LLM agent on the \texttt{ALFWorld} dataset after training with conditional state-space exploration with intrinsic rewards.}
  \label{fig:qualitative_appendix}
\end{figure*}

\begin{figure*}[t]
\centering
  \includegraphics[width=\linewidth]{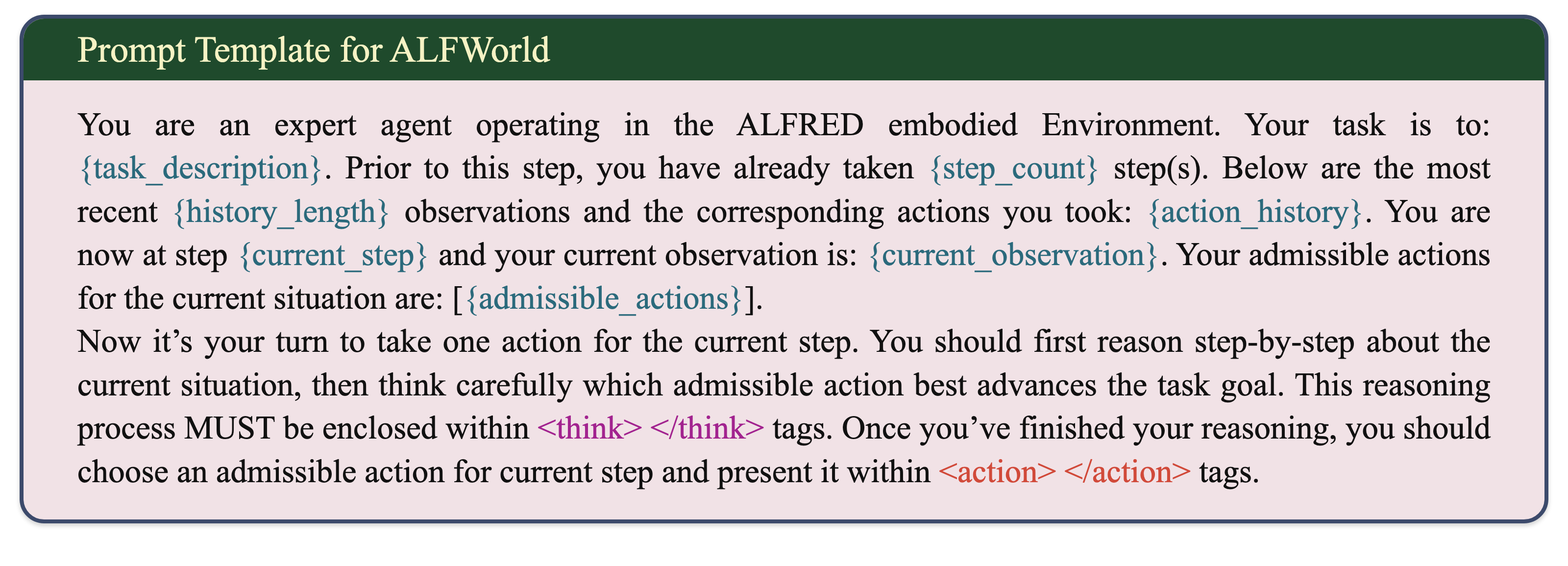}
  \vspace{6pt}
  \includegraphics[width=\linewidth]{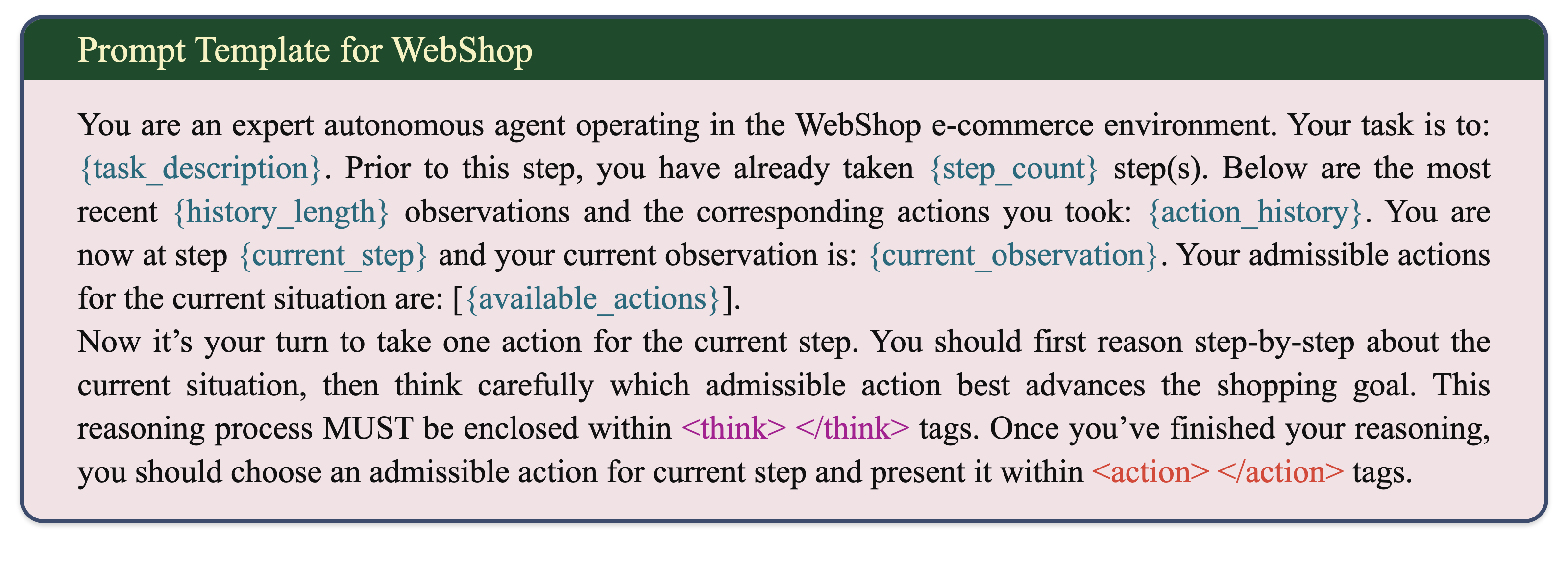}
  \caption {Prompts used for both \texttt{ALFWorld} and \texttt{WebShop} environments.}
  \label{fig:prompts}
\end{figure*}

\subsection{Reproducibility Details}
\label{sec:reproducibility}
\paragraph{Penalty weight $\gamma$ and reward weight $\lambda$.} $\gamma=1/10$ destabilized training; $\gamma=1/100$ matched the baseline; $\gamma=1/40$ (used throughout) gave a clear, consistent improvement.

\paragraph{Prompts.} Shown in Figure~\ref{fig:prompts}.

\paragraph{Invalid actions:} The two environments signal invalid actions differently. In \texttt{ALFWorld}, every action whose preconditions are not satisfied returns the fixed observation \texttt{"Nothing happens."}, and we count a step as invalid when its observation matches this string or when the agent's output contains no well-formed action. \texttt{WebShop} provides no failure message: an action that is not available on the current page, such as clicking an element that is not displayed, is ignored, and the environment returns the unchanged page. For \texttt{WebShop}, we therefore count a step as invalid when the agent's output contains no well-formed action. Because an ignored action returns the same page, it always produces a duplicate observation and is penalized by our intrinsic reward.
 
\paragraph{Rollout and optimization settings.} Batch size, rollouts/step, PPO epochs, clip range, and learning rate are unmodified \emph{verl-agent}/GiGPO defaults~\citep{feng2026group}. Details are mentioned in Table~\ref{tab:hparams}.

\begin{table}[h]
\centering
\caption{Training hyperparameters.}
\label{tab:hparams}
\begin{tabular}{ll}
\toprule
Hyperparameter & Value \\
\midrule
Batch size & 256 \\
Rollouts per step & 128 \\
PPO clip range & 0.2 \\
Learning rate & 1e-6 \\
Training steps & 300 \\
Max steps (\texttt{ALFWorld}) & 50 \\
Max steps (\texttt{WebShop}) & 15 \\
$\gamma$ (intrinsic penalty) & $1/40 \times$ success reward \\
\bottomrule
\end{tabular}
\end{table}

\paragraph{Train/test splits.} We use the standard \texttt{ALFWorld} seen/unseen split released by \citet{shridhar2020alfworld} and the standard \texttt{WebShop} train/eval split released by \citet{yao2022webshop}.

\subsection{Declaration}
We have cited the models and code used in Section~\ref{sec:experiments}. We have used ChatGPT, Gemini, and Claude to assist with polishing writing.

\end{document}